\documentclass{article}

\usepackage{microtype}
\usepackage{graphicx}
\usepackage{subcaption}
\usepackage{bm}
\usepackage{xfrac}
\usepackage{booktabs} 
\usepackage[normalem]{ulem}
\usepackage{hyperref}

\usepackage[accepted]{icml2023}

\usepackage{amsmath}
\usepackage{amssymb}
\usepackage{mathtools}
\usepackage{amsthm}
\usepackage{thmtools} 
\usepackage{thm-restate}

\usepackage[capitalize,noabbrev]{cleveref}

\theoremstyle{plain}
\newtheorem{theorem}{Theorem}[section]

\theoremstyle{definition}
\newtheorem{definition}[theorem]{Definition}

\theoremstyle{remark}

\newcommand\bigforall{\mbox{\Large $\mathsurround0pt\forall$}} 
\usepackage[textsize=tiny]{todonotes}

\icmltitlerunning{Are Random Decompositions all we need in High Dimensional Bayesian Optimisation?}

\begin{document}

\twocolumn[
\icmltitle{Are Random Decompositions all we need in High Dimensional Bayesian Optimisation?}



\icmlsetsymbol{equal}{*}

\begin{icmlauthorlist}
\icmlauthor{Juliusz Ziomek}{huawei}
\icmlauthor{Haitham Bou-Ammar}{huawei}

\end{icmlauthorlist}

\icmlaffiliation{huawei}{Huawei Noah’s Ark Lab, London, UK}

\icmlcorrespondingauthor{Haitham Bou-Ammar}{haitham [dot] ammmar (at) huawei \{dot\} com}

\icmlkeywords{Machine Learning, ICML}

\vskip 0.3in
]



\printAffiliationsAndNotice{}  
\begin{abstract}

Learning decompositions of expensive-to-evaluate black-box functions promises to scale Bayesian optimisation (BO) to high-dimensional problems. However, the success of these techniques depends on finding proper decompositions that accurately represent the black-box. While previous works learn those decompositions based on data, we investigate  data-independent decomposition sampling rules in this paper. We find that data-driven learners of decompositions can be easily misled towards local decompositions that do not hold globally across the search space. Then, we formally show that a random tree-based decomposition sampler exhibits favourable theoretical guarantees that effectively trade off maximal information gain and functional mismatch between the actual black-box and its surrogate as provided by the decomposition. Those results motivate the development of the random decomposition upper-confidence bound algorithm (RDUCB) that is straightforward to implement - (almost) plug-and-play - and, surprisingly, yields significant empirical gains compared to the previous state-of-the-art on a comprehensive set of benchmarks. We also confirm the plug-and-play nature of our modelling component by integrating our method with HEBO \cite{cowen2022hebo}, showing improved practical gains in the highest dimensional tasks from the Bayesmark problem suite.  
\end{abstract}

\section{Introduction}
Although Bayesian optimisation (BO) demonstrated impressive successes in low-dimensional domains \cite{marchant2012bayesian,shahriari2015taking,kandasamy2017multi,kandasamy2018parallelised, grosnit2022boils}, scaling BO to high-dimensional and expensive to evaluate black-box functions has proved challenging. Among the many proposed approaches ranging from linear to non-linear projections \cite{wang2016bayesian,rana2017high, li2018high, tripp2020sample, moriconi2020high, eriksson2021high, grosnit2021high, wan2021think}, decomposition methods that assume additively structured black-boxes emerged as a promising direction for high-dimensional BO \cite{kandasamy2015high,rolland2018high,han2021high}.     

Given a decomposition in the dimensions of the problem, those techniques utilise additive kernel Gaussian processes (GPs) as surrogate models to trade off exploration and exploitation when suggesting novel queries to evaluate. Additive techniques uncover new inputs by maximising an acquisition function that is additive under the provided decomposition. Although first and second-order methods \cite{wilson2017reparameterization} can be used to maximise acquisition functions, we adopt message-passing optimisers that can better exploit additive acquisition structures \cite{rolland2018high}. 

The success of additive methods in high-dimensional BO depends on the correct choice of decompositions that must accurately mimic the inter-dimensional dependencies of the actual black-box function. Prior art empirically demonstrated that tree-structured decompositions (i.e., cycle-free pair-wise dimensional interactions) could effectively represent many black-box functions. The Tree algorithm \cite{han2021high} discovers the best tree decomposition for the black-box based on the data collected during BO, by maximising a new GP marginal that includes decomposition parameters (encoding sparsity) and length scales. 

At first glance, learning decompositions based on data and marginal likelihood is plausible and can yield promising optimisation results. While this is true  when given a fixed dataset, dynamically acquired data during BO provides, at best, local (within the probed regions) function information, making it challenging to extrapolate dimensional interdependencies across the search space. In other words, such agents are easily misled by modifying the black-box function's factorisation based on regions of the search space (for example see Section \ref{Sec:Th}). 

{\textbf{Contributions:}} Of course, one can think of many directions to resolve this issue, like analysing distribution shifts \cite{kirschner2020distributionally}, maximising estimators beyond marginals \cite{ziegel2003elements}, or even designing novel Gaussian process kernels. In our work, however, we prefer to develop a simple (almost) plug-and-play approach that leads to empirical gains while adhering to rigorous theoretical guarantees. Therefore, rather than relying on data-driven learning, we \emph{investigate adopting data-independent decomposition rules}. Our theoretical results indicate that random decomposition sampling strategies achieve the lowest expected mismatch to the black-box function when we fix the class of decompositions to trees, which allows us to favourably bound maximal information gain. Equipped with these results, we then propose the random decomposition upper-confidence bound (RDUCB) algorithm. RDUCB utilises a random tree sampler and an additive acquisition function to achieve superior empirical performance on a broad set of benchmarks compared to the prior state-of-the-art. The modelling component of RDUCB is simple to implement and thus easy to integrate on top of many existing BO frameworks. We support this claim by augmenting HEBO \cite{cowen2022hebo} - the winning submission of the NeurIPS 2020 black-box optimisation challenge \cite{turner2021bayesian} - with our random decomposition sampling strategy. We demonstrate that this adaptation, which we title RDHEBO, delivers improved performance on the set of highest dimensional Bayesmark \cite{asuncion2007uci, turner2021bayesian} tasks.    

\section{Background}
\subsection{Bayesian Optimisation (BO)} \label{Sec:BO}
We employ a sequential decision-making approach to the maximisation of expensive-to-evaluate black-box functions $f:\mathcal{X} \rightarrow \mathbb{R}$ over an input domain $\mathcal{X} \subseteq \mathbb{R}^d$. At each round $t$, we determine an input $\bm{x}_{t} \in \mathcal{X}$ and observe its black-box function value $f(\bm{x}_t)$. We allow noise-corrupted observations such that $y_t \sim \mathcal{N}(\bm{x}_t, \sigma_{\text{n}}^2)$. Our goal is to approach the optimum, $\bm{x}^{\star} \equiv \arg\max_{\bm{x}\in\mathcal{X}} f(\bm{x})$, rapidly as a function of black-box queries. Given that both the black-box function and optima are unknown, BO solvers trade-off exploration and exploitation via a two-step procedure involving fitting a surrogate model and maximising an acquisition function. We detail each of those steps below. 

{\textbf{Gaussian Process Surrogates (GPs):}} GPs allow us to place priors directly in the function space by specifying a mean function $m(\bm{x})$ and a covariance kernel $k(\bm{x}, \bm{x}^{\prime})$ that encode our assumptions about the black-box. Following \cite{williams2006gaussian}, we assume a zero-mean $m(\bm{x}) \equiv 0$ and adopt a squared exponential covariance kernel:
\begin{equation*}
    k(\bm{x}, \bm{x}^{\prime}) = \exp\left(-\frac{1}{2}(\bm{x} - \bm{x}^{\prime})^{\mathsf{T}} \text{diag}(\bm{\theta}^{2})^{-1}(\bm{x} - \bm{x}^{\prime})\right),
\end{equation*}
where $\bm{\theta}$ is a set of hyper-parameters tuned by maximising the data marginal\footnote{In our equation, we execute $\bm{\theta}^2$ element-wise.}. 

Given the data $\mathcal{D}_t = \{\bm{x}_{i}, y_{i}\}_{i=1}^t$ collected so-far during $t$ BO rounds, we write the posterior at a point $\bm{x}$ as $p(f(\bm{x})|\mathcal{D}_t) \sim \mathcal{N}(\mu_t(\bm{x}), \sigma_{t}(\bm{x}))$ with:
\begin{align*}
    \mu_t(\bm{x}) &= \bm{k}_t^{\mathsf{T}}(\bm{x})\left(\bm{K}_t + \sigma_{\text{n}}^2\bm{I}\right)^{-1} \bm{y}_t \\ 
    \sigma_t(\bm{x})^2 &= k(\bm{x},\bm{x}) - \bm{k}_t^{\mathsf{T}}(\bm{x})\left(\bm{K}_t + \sigma_{\text{n}}^2\bm{I}\right)^{-1}\bm{k}_t (\bm{x}),
\end{align*}
where $\bm{y}_{t}$ concatenates all observations $y_{1:t}$ in one vector and $\bm{I}$ represents an identity matrix. The matrix $\bm{K}_t$ evaluates the covariance kernel on all input pairs in $\mathcal{D}_t$. Finally, the vector $\bm{k}_t$ contains the kernel evaluation between $\bm{x}$ and all input data points from $\mathcal{D}_t$. 

{\textbf{Acquisition Function Maximisation:}} Given the posterior predictive distribution above, we now discuss how to suggest novel query points that improve the guess of the optima. In BO, this process involves maximising an acquisition function $\alpha(\cdot|\mathcal{D}_t)$ that utilises the probabilistic model such that $\bm{x}_{t} \in \arg\max_{\bm{x}\in \mathcal{X}}\alpha_t(\bm{x}|\mathcal{D}_{t-1})$. While there exist many acquisitions ranging from myopic to non-myopic forms \cite{frazier2008knowledge,grosnit2021we, cowen2022hebo, shahriari2015taking, zhang2021constrained}, in this paper, we follow \cite{srinivas2009gaussian,srinivas2012information} and adapt the upper-confidence bound as the function of choice: 
\begin{equation*}
    \alpha_t^{(\text{UCB})}(\bm{x}|\mathcal{D}_{t-1}) = \mu_{t-1} (\bm{x}) + \beta_t \sigma_{t-1}(\bm{x}), 
\end{equation*}
where $\beta_t$ is a tuneable hyperparameter.


\subsection{High-Dimensional BO with Decompositions}
BO in high-dimensional spaces is an active area of research. In Section \ref{Sec:RelatedWork}, we survey related methods to our work. In this section, we focus on decomposition-based strategies that promise to scale BO while accelerating acquisition optimisation. Before diving into the details of decomposition-based techniques, we now introduce the notion of decompositions of $d$-dimensional spaces as follows: 
\begin{definition}
\textit{A decomposition $g$ of $d$-dimensions is a collection of sets $c$, called components, consisting of dimensions $i \in [1:d]$, i.e. $\forall_{c \in g}\forall_{i \in c} i \in [1:d]$. }
\end{definition}
Decomposition-based BO assumes that the black-box function decomposes according to $g$: 
$f(\bm{x}) = \sum_{c \in g} f_c\left(\bm{x}_{[c]}\right)$, where  $\bm{x}_{[c]}$ selects those dimensions of $\bm{x}$ that appear in $c$.  

{\textbf{Additive GP Kernels:}} Compared to standard BO from Section \ref{Sec:BO}, the first change decomposition methods employ is the usage of additive kernels \cite{duvenaud2011additive,durrande2012additive,qamar2014additive,lu2022additive} that better suit a decomposable black-box function: $k^g(\bm{x}, \bm{x'}) = \sum_{c \in g} k_c\left(\bm{x}_{[c]}, \bm{x'}_{[c]}\right)$. Significantly, if two dimensions $i$ and $j$ do not appear together in any of the sets $c$, the kernel will not model interactions between them. \citet{rolland2018high} showed that the posterior of each component subfunction $f_c\left(\bm{x}_{[c]}\right)$ can be expressed as $p \left(f_c\left(\bm{x}_{[c]}\right)| \mathcal{D}_t \right) = \mathcal{N}\left(\mu_{t,c}\left(\bm{x}_{[c]}\right), \sigma^{2}_{t,c}\left(\bm{x}_{[c]}\right) \right)$, where:
\begin{align*}
    \mu_{t,c}\left(\bm{x}_{[c]}\right) &= \bm{k}^{\mathsf{T}}_{t,c}\left(\bm{x}_{[c]}\right) (\bm{K}_t + \sigma_{\text{n}}^2\bm{I})^{-1}\bm{y}_t\\
    \sigma^{2}_{t,c}\left(\bm{x}_{[c]}\right) &= k_{c}\left(\bm{x}_{[c]}, \bm{x}_{[c]}\right) \\ & \hspace{3.5em}- \bm{k}_{t,c}^{\mathsf{T}}\left(\bm{x}_{[c]}\right)(\bm{K}_t+\sigma_{n}^2\bm{I})^{-1}\bm{k}_{t,c}\left(\bm{x}_{[c]}\right),
\end{align*}
where $\bm{k}_{t,c}(\bm{x}_{[c]})$ is a vector of kernel evaluations between $\bm{x}$ and all inputs in $\mathcal{D}_t$ while only considering those dimensions $i \in c$ that appear in $c$. 

If the size of each set $c$ is much smaller than the total dimensionality $d$, such an approach will enjoy many benefits. First, we escape the curse of dimensionality, as we only need to consider interactions between a small number of dimensions. Second, if we utilise an acquisition function with an additive structure, such as the additive UCB: $ \alpha^{(\text{add-UCB})}_t(\bm{x}|\mathcal{D}_{t-1}) = \sum_{c \in g} \alpha_{t,c}(\bm{x}|\mathcal{D}_{t-1}) = \sum_{c \in g} \mu_{t-1, c}(\bm{x}) + \beta_t \sigma_{t-1, c}(\bm{x})$ we can determine novel query points to evaluate very efficiently using message passing \cite{rolland2018high}.

However, this is only possible if we know the decomposition of the black box. Existing methods attempt to learn it from data using maximum likelihood \cite{kandasamy2015high,rolland2018high, han2021high}, but there are no guarantees regarding the correctness of such a learning procedure. We expand on the problems associated with this approach in the next section.

\section{Misleading Decomposition Learners}\label{Sec:Th}
As noted in the previous section, decomposition methods learn the optimal additive kernel structure by selecting the one which produces a model with the highest marginal likelihood. While plausible and generally adopted, we now point out some problems associated with likelihood maximisation when learning decompositions in BO. Before we underpin this issue from a theoretical perspective, we first provide an empirical example that demonstrates those challenges next.   

\subsection{Challenges to Decompositions Learners} 
Consider the maximisation of the function demonstrated by the heat plot shown in Figure \ref{fig:toyproblem}. It is easy to see that while this function is not fully separable due to the correlated mode in the top-right corner, it can locally appear as if it was entirely separable. 

Imagine that we collected several initial points from the region $0 \le x \le 600$ and $0 \le y \le 600$. Methods that rely entirely on data when learning the kernel structure (e.g., by maximising marginals) would exploit the local view of the function and falsely believe a complete separation in dimensions. We support this realisation in Figure \ref{fig:toyproblem} by running the state-of-the-art Tree algorithm from \cite{han2021high} that learns decompositions via maximum likelihood.    

\begin{figure}[h!]
    \centering
    \includegraphics[trim={8em 2em 9em 3em}, clip=true, width=\columnwidth]{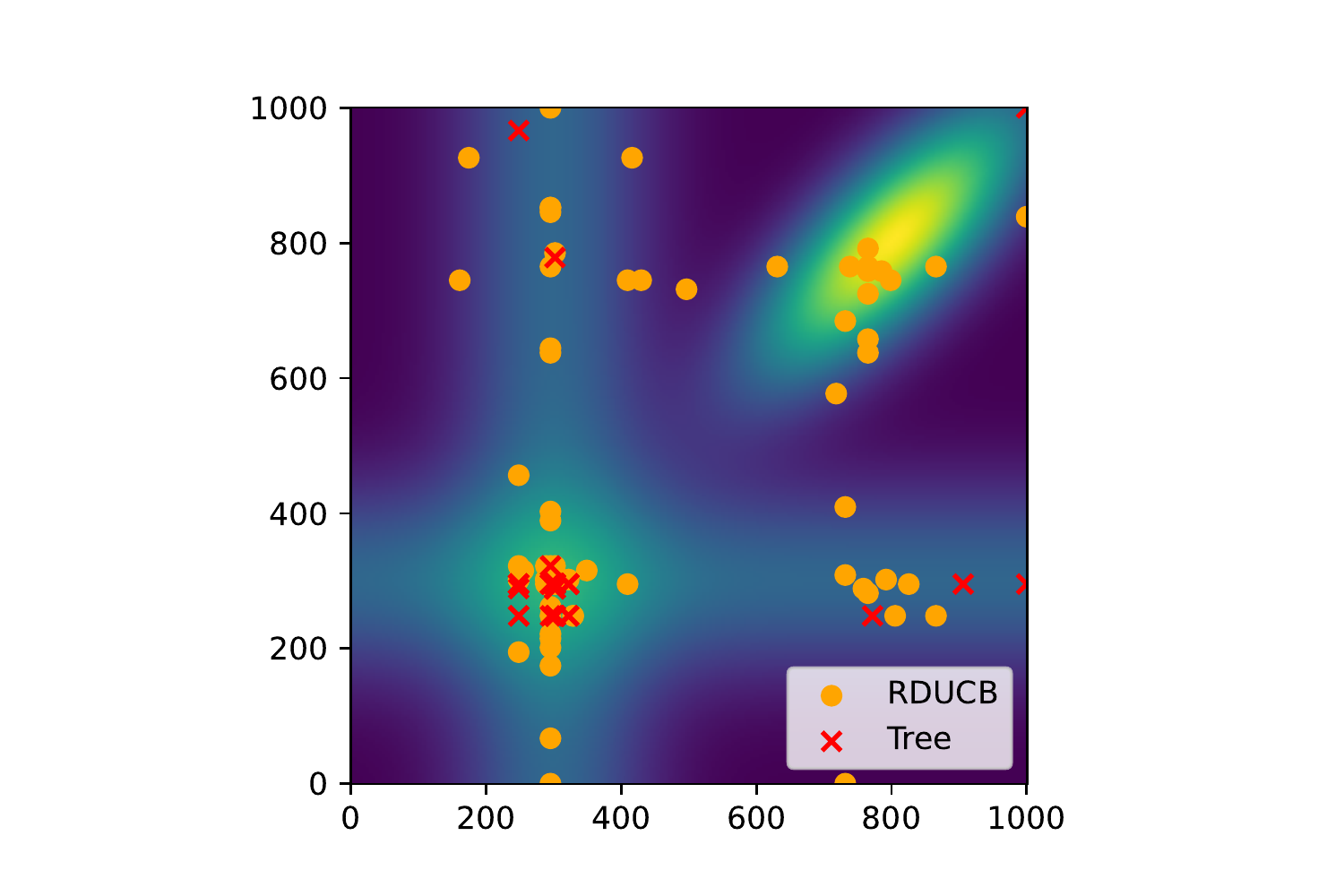}
    \caption{Comparison of points queried on a toy problem, given the same initial design. Tree refers to the method from \cite{han2021high}, while RDUCB is the algorithm we propose in this paper. It is clear that Tree gets stuck in the local mode, while RDUCB eventually arrives at the optimal mode. For details see App. \ref{App:toy_problem}.}
    \label{fig:toyproblem}
\end{figure}
We see that the Tree algorithm  gets stuck in the local mode because of erroneously learning that the function is fully decomposable. In contrast, our algorithm RDUCB, which we propose throughout the paper, circumvents this problem, whereby given the same initial design and settings of Tree, it can still find the optimal mode.



\subsection{What Causes the Failure?} \label{Sec:Failure}
To better understand the above failure mode, we notice that approaches learning decompositions from data (e.g., Tree \cite{han2021high}) assume that the function abides by one decomposition that does not vary across the search space, which is not the case in Figure \ref{fig:toyproblem}. 

To improve decomposition methods in high-dimensional BO, we wish to develop algorithms that enable varying decompositions across the search space. We could formalise this problem by imagining an adversary choosing a decomposition, $g$, and a corresponding black-box function, $f(\cdot)$, to optimise. This function must be a member of the reproducing kernel Hilbert space $\mathcal{H}^g$ defined by some kernel $k^g(\bm{x},\bm{x}^{\prime})$, following the selected decomposition $g$, unknown to the algorithm. It is important to note that such an adversary can select a black-box function that globally follows a decomposition $g$ but locally appears to have fewer interaction components than $g$. As we cannot rely on  locally collected data, we wish to investigate data-independent rules while ensuring a rigorous theoretical understanding.

{\textbf{On Data Independent Kernel Updates:}} Although we are the first to propose data-independent updates to select decompositions in BO, it is worth noting that the work in \cite{Berkenkamp} already considered data-independent rules but when tuning kernel hyper-parameters with no focus on decompositions. The authors demonstrated no-regret bounds for an algorithm that alters the GP kernel's hyperparameters according to a predefined scheme that does not rely on data gathered during BO. 

However, the direct application of the work in \cite{Berkenkamp} to learning decompositions is challenging for several reasons. First, our problem setting varies in that the authors in \cite{Berkenkamp} consider low-dimensional BO that does not require learning kernel decompositions. Second and more importantly, the update rule in \cite{Berkenkamp} keeps decreasing the length scales, expanding the kernel's complexity as measured by maximum information gain \cite{srinivas2009gaussian}, which will play a critical role in trading-off complexity versus mismatch as we note in Section \ref{Sec:AnalyseDecom}. 


Next, we expand on how to design such data-independent decomposition schemes and note that sparse random trees serve as a simple, effective, and scalable strategy. 

\section{Decompositions Without Learning} \label{Sec:Decompositions}
To better understand what properties constitute good predefined decomposition rules, we begin with a theoretical study that hints at the necessary trade-offs our strategy needs to optimise. Let us introduce $S(t): \mathbb{Z}^+ \to \mathcal{G}$ to be a predefined (data-independent) scheme that selects a decomposition from some class of decompositions $\mathcal{G}$ at each round $t$. Consequently, $S(t)$ defines the kernel $k_t (\bm{x},\bm{x}^{\prime}) = k^{g_t} (\bm{x},\bm{x}^{\prime})$ that we use during that round $t$. 

We wish to derive a high-probability regret bound in terms of 1) a quantity roughly measuring the kernel's complexity and 2) a notion of function mismatch between the true black-box and those functions spanned by our kernel $k_t(\bm{x},\bm{x}^{\prime})$. For kernel complexity, we follow \cite{srinivas2009gaussian} and use the maximum information gain $\gamma_T$ of kernels defined by the decompositions proposed by our scheme\footnote{Of course in \cite{srinivas2009gaussian} decompositions are not considered. Here, we define a slight generalisation of maximal information gain to handle the case of a changing kernel.}. Here, $\gamma_{T} = \max_{\bm{X} \subseteq \mathcal{X}, |\bm{X}| = T} I(\bm{f}_T, \bm{y}_{T})$ denotes the maximum information gain after $T$ steps, where $\bm{X}$ is a set of $T$ selected points and $\bm{f}_T = (\hat{f}_1(\bm{x}_1), \dots, \hat{f}_T(\bm{x}_T))$, where $\hat{f}_t = \sum_{c \in g_t} f_c$ and  $f_c \sim \mathcal{GP}(0,k_c(\bm{x},\bm{x}^{\prime}))$. Concerning function mismatch, we define $\epsilon_t = |\hat{f}_t - f|_{\infty}$ such that $\hat{f}_t = \arg\min_{f'\in \mathcal{H}_{t}}|f' - f|_{\infty}$ is the function from the reproducing kernel Hilbert space (RKHS) $\mathcal{H}_{t}$ of kernel $k_t (\bm{x},\bm{x}^{\prime})$ that is closest to the black-box $f(\cdot)$ in terms of infinity norm. This way $\epsilon_t$ provides a notion of mismatch between the actual black-box and the closest function from the RKHS of $k_{t}(\bm{x},\bm{x}^{\prime})$. Now, if we run a UCB-style BO algorithm for $T$ rounds, we obtain the following result.
\begin{restatable}[]{theorem}{regretbound} \label{th:regretbound}
Let the black-box function $f$ be selected by an adversary from an RKHS $\mathcal{H}^g$ of kernel $k^g$, defined over some decomposition $g$ that is also selected  by an adversary. After $T$ rounds, a UCB-style BO algorithm with an  $S(t): \mathbb{Z}^+ \to \mathcal{G}$ decomposition rule, incurs with a probability of at least $1 - \delta$ the following total cumulative regret $R_T$:
\begin{align*}
    R_T =   \mathcal{O} \left ( \sqrt{T\gamma_T} \left( B  + \sqrt{\ln\frac{1}{\delta} + \gamma_T }   + \sum_{t=1}^T \epsilon_t \right)\right ) ,
\end{align*}
where $B = \max_{t \in T} \lVert \hat{f}_t \rVert_{t}$ and $\lVert \cdot \rVert_{t}$ denotes the norm in $\mathcal{H}_{t}$.
\end{restatable}
\begin{proof}(Sketch)
We defer complete proof to Appendix \ref{ap:proof_regretbound}. Here, we provide a sketch of the main steps. BO under misspecification has been studied by \cite{bogunovic2021misspecified}. In their setting, the authors assume fixed kernels in-between iterations. Hence, we need to adapt the proof from \cite{bogunovic2021misspecified} to our setting where the kernel and the mismatch vary. To do so, we observe that at any given time, the difference between $\max_{\bm{x} \in \mathcal{X}}\hat{f}_t(\bm{x})$ and $\max_{\bm{x} \in \mathcal{X}}f(\bm{x})$ can be at most $\epsilon_t$. Consequently, the difference between $\max_{\bm{x} \in \mathcal{X}}\hat{f}_t(\bm{x})$ and $\hat{f}_t(\bm{x}_t)$ becomes the new term that we bound. Here, we adapt a high probability bound to the case of a changing kernel. The final step of the proof is to express the bound in terms of maximum information gain $\gamma_T$, which easily follows from preceding BO literature  \cite{srinivas2009gaussian}.
\end{proof}

Note that $\epsilon_t$ is a random quantity. Thus to make this bound easier to analyse, we derive the following Corollary. We provide its proof in Appendix \ref{ap:proof_corrregret}.
\begin{restatable}[]{corollary}{finalregretbound}
\label{Corr:Regret}
    Under the assumptions of Theorem \ref{th:regretbound}, we have with probability at least $1 - \delta_A - \delta_B$, the cumulative regret $R_T$ of a UCB-style BO algorithm utilising the decomposition suggesting scheme $S(t)$, incurs the following cumulative regret:
\begin{align*}
    R_T =   \mathcal{O} \Bigg ( \sqrt{T\gamma_T} \Bigg(
   B  &+ \sqrt{\ln\frac{1}{\delta_A} + \gamma_T } \\
   &+ \frac{\mathbb{E}_S \left[\sum_{t=1}^T  \epsilon_t \right] }{\delta_B} \left)\rule{0cm}{1cm}\right ) .
\end{align*}
\end{restatable}

\subsection{Analysing Decomposition Rules}\label{Sec:AnalyseDecom}
From Corollary \ref{Corr:Regret}, we notice that we need a decomposition rule ${S}(t)$ such that both $\gamma_T$ and $\mathbb{E} \left[\sum_{t=1}^T \epsilon_t\right] $ are ``small''. To bound  $\gamma_T$, we require a restriction on the class of decompositions our scheme can propose. If we choose this class to consist only of decompositions with pairwise components, we get the following result on $\gamma_{T}$, proven in Appendix \ref{ap:proof_informationgain}. 
 \begin{restatable}[]{proposition}{informationgain} \label{prop:informationgain}
  The maximum information gain for a squared exponential kernel following a decomposition only with pair-wise components in $d$ dimensions is upper-bounded by $
     \gamma_T \le \mathcal{O}(d(\log T )^ 3)
$. 
 \end{restatable} 
 In light of the result above, it might be tempting to try a decomposition with all possible pairwise components. However, given the empirical success of the Tree algorithm \cite{han2021high}, it appears that for many problems, it is sufficient to consider tree-structured decompositions, which we now proceed to define.
\begin{definition}
    \textit{A decomposition $g$ is  tree-based with $E$ edges if it has $E$ pair-wise components and the remaining components contain one dimension. Additionally, an undirected graph formed by edges corresponding to pair-wise interactions does not contain a cycle.}
\end{definition}

  If we restrict ourselves to the class of trees, we additionally reduce the maximum information gain, as stated by the next result, proven in Appendix \ref{ap:proof_lowerinformation}.
  \begin{restatable}[]{proposition}{lowerinformation} \label{prop:lowerinformation}
  When $d > 2$, the maximum information gain for a kernel following a tree decomposition (possibly changing between timesteps) is always smaller than for a kernel containing all pairwise interactions. 
 \end{restatable} 
{\textbf{Bounding Expected Mismatches $\mathbb{E} \left[\sum_{t=1}^T \epsilon_t\right] $:}} Bounding $\mathbb{E} \left[\sum_{t=1}^T \epsilon_t\right] $ is far more challenging. To understand this difficulty, let us consider two extremes when suggesting decompositions: constant and adaptive decomposition selection rules $S(t)$. If $S(t)$ constantly suggests the same decomposition, the adversary efficiently exploits this strategy resulting in a constant high mismatch. While if we adaptively increase the kernel's complexity using procedures like those in \cite{Berkenkamp}, our method will indeed reduce mismatch regardless of the adversary's choices. However, as we introduce more interactions between dimensions, our decompositions eventually stop being trees and thus lead to increases in information gain $\gamma_T$. 


Analysing the bound above, we observe a trade-off between maximal information gain and the risk of being exploited by an adversary. To resolve this problem, we propose to use a randomised tree sampling strategy that keeps the information gain small while circumventing adversaries. Importantly, we show that such an $S(t)$ exhibits the lowest expected mismatch within a class of tree-structured decompositions. We now state the formal result.

\begin{restatable}[]{theorem}{optimalscheme} \label{th:optimalscheme}
Let $\mathcal{G}$ be a class of $d$-dimensional tree-based decompositions with $E$ edges. Let the adversary choose any function $f$ from RKHS $\mathcal{H}_{g}$ defined by kernel $k_g$ on a decomposition $g \in \mathcal{G}$ also chosen by the adversary, such that the infinity norm of each component function is bounded $\forall_{c \in g}|f_c|_{\infty} \le M_c$ and $\sum_{c \in g} M_c \le M$. We then have that for any data-independent scheme $S$, the expected sum of mismatches is at least as large as for the scheme $S_{\text{r}}$ that selects a decomposition from $\mathcal{G}$ uniformly at random, i.e.,
\begin{equation*}
    \bigforall_{S : \mathbb{Z}^+ \to \mathcal{G}} \quad \mathbb{E}_{S} \left[\sum_{t=1}^T \epsilon_t \right] \ge \mathbb{E}_{S_\text{r}} \left[\sum_{t=1}^T \epsilon_t \right] .
\end{equation*}

\end{restatable}

 \begin{proof}(Sketch) We defer complete proof to Appendix \ref{ap:proof_optimalscheme}. Observe that the optimal strategy for an adversary is to put all the mismatch $M_c$ on the pair-wise component $c$ least selected by the scheme $S(t)$. Since we will suffer a mismatch of $M$ every time we do not select this component,  the optimal scheme needs to maximise the expected number of times the least selected component is chosen. This happens when all pair-wise components have the same probability of being selected, proving the optimality of $S_{\text{r}}$. 
 \end{proof}

\subsection{Practical Algorithm: Random Decompositions UCB}
Equipped with the above results, this section presents a practical and effective algorithm for high-dimensional BO. Our algorithm, titled random decomposition upper confidence bound (RDUCB), adheres to the theoretical results from the previous section and allows a scalable implementation. 
\begin{algorithm}[h!]
   \caption{RDUCB}
   \label{alg:rd_ucb}
\begin{algorithmic}[1]
   \STATE {\bfseries Inputs:} Black-box function $f$, evaluation budget $N$, initial budget $N_{\text{init}}$, exploration bonuses $\{\beta_t\}_{t=1}^{N}$
   \STATE Evaluate $N_{\text{init}}$ random inputs in $f$ \& populate  $\mathcal{D}_{N_{\text{init}}}$
   \FOR{$t=N_{\text{init}} + 1$ {\bfseries to} $N$}
        \STATE Sample tree decomposition $g$ (Alg. \ref{alg:samplingtrees})
        \STATE Fit a GP using $\mathcal{D}_{t-1}$ with the kernel $k_g(\cdot)$
        \STATE Maximise $\alpha^{(\text{add-UCB})}_t(\bm{x}|\mathcal{D}_{t-1})$ with message passing  
   \STATE Evaluate $f$ on the suggested query \& add to $\mathcal{D}_{t-1}$
   \ENDFOR
\end{algorithmic}
\end{algorithm}

Algorithm \ref{alg:rd_ucb} is a pseudo code of RDUCB. At a high-level, our method follows any generic BO solver in that it first fits a probabilistic model (line 5) and then maximises an acquisition, as shown in line 6. Of course, we maximise additive UCB acquisitions as dictated by the sampled decomposition $g$. As noted earlier, decompositions $g$ correspond to trees that we sample so that each edge has an equal probability of being selected. We present this sampler (referred to in line 4 of Algorithm \ref{alg:rd_ucb}) in Algorithm \ref{alg:samplingtrees} in Appendix \ref{ap:treesampler}.



\begin{figure*}[t!]
\begin{subfigure}[b]{0.33\textwidth}
  \centering
   \includegraphics[trim={0em 0em 5em 2em}, clip=true, width=\linewidth]{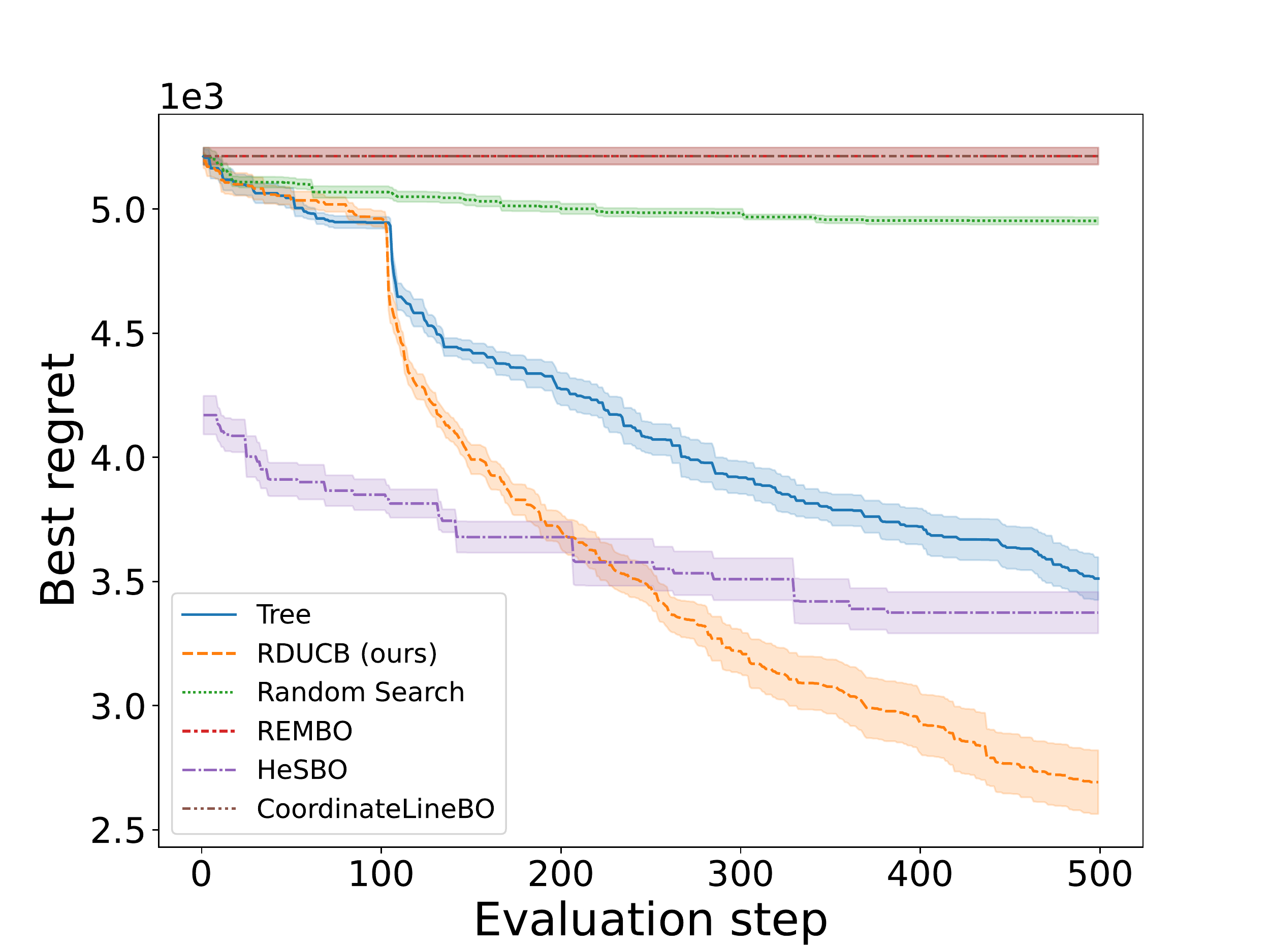}
 \caption{250-d Stybtang Function}
\end{subfigure}%
\begin{subfigure}[b]{0.33\textwidth}
  \centering
   \includegraphics[trim={0em 0em 5em 2em}, clip=true,width=\linewidth]{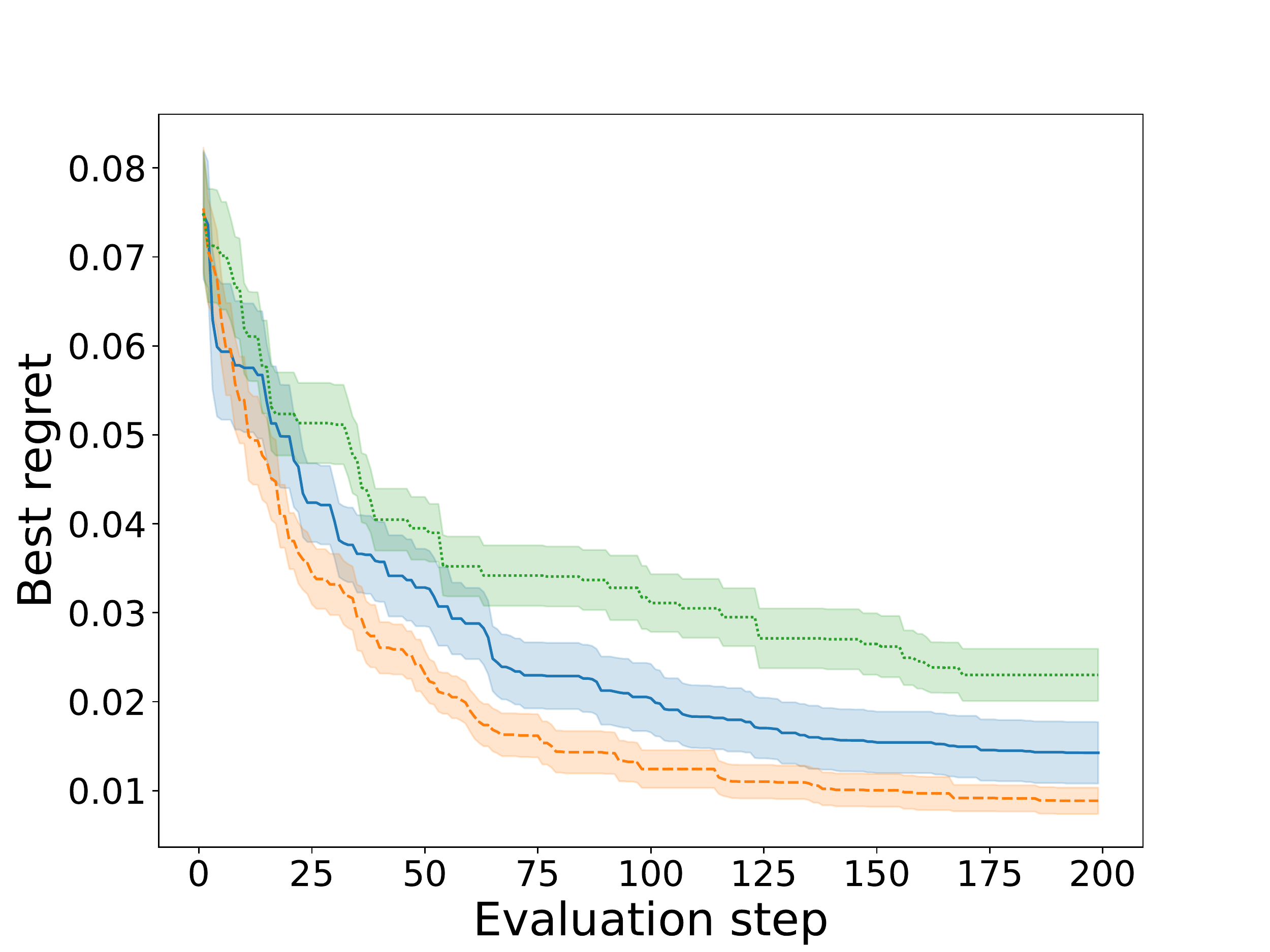}
 \caption{9-d NAS Protein Dataset}
\end{subfigure}%
    \begin{subfigure}[b]{0.33\textwidth}
  \centering
   \includegraphics[trim={0em 0em 5em 2em}, clip=true,width=\linewidth]{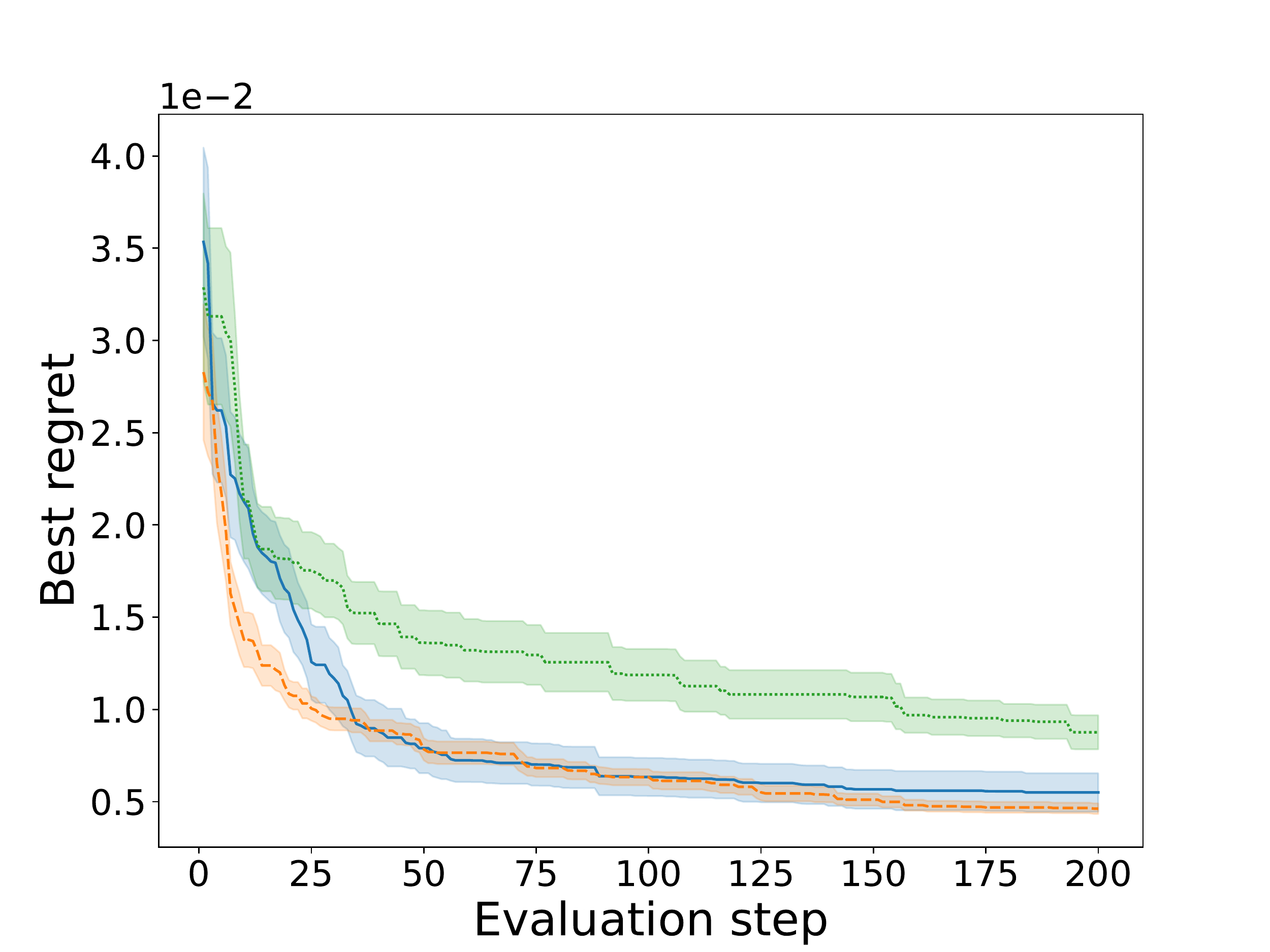}
 \caption{9-d NAS Parkinson Dataset}
\end{subfigure}%
\caption{Best regret results from the 250-dimensional synthetic functions (averaged over 10 random seeds) and the NAS benchmark (averaged over 20 random seeds). We note that REMBO is not reported for the NAS (Protein and Parkinson) datasets as those involve discrete variables. Tree refers to the previous state-of-the-art decomposition technique from \cite{han2021high}. 
We observe that RDUCB outperforms Tree, REMBO and random search. It is also clear that as the dimensions increase, so does the performance of our algorithm.}
 \label{fig:stybtang_protein_parkinson}
\begin{subfigure}[b]{0.33\textwidth}
  \centering
   \includegraphics[trim={0em 0em 5em 2em}, clip=true,width=\linewidth]{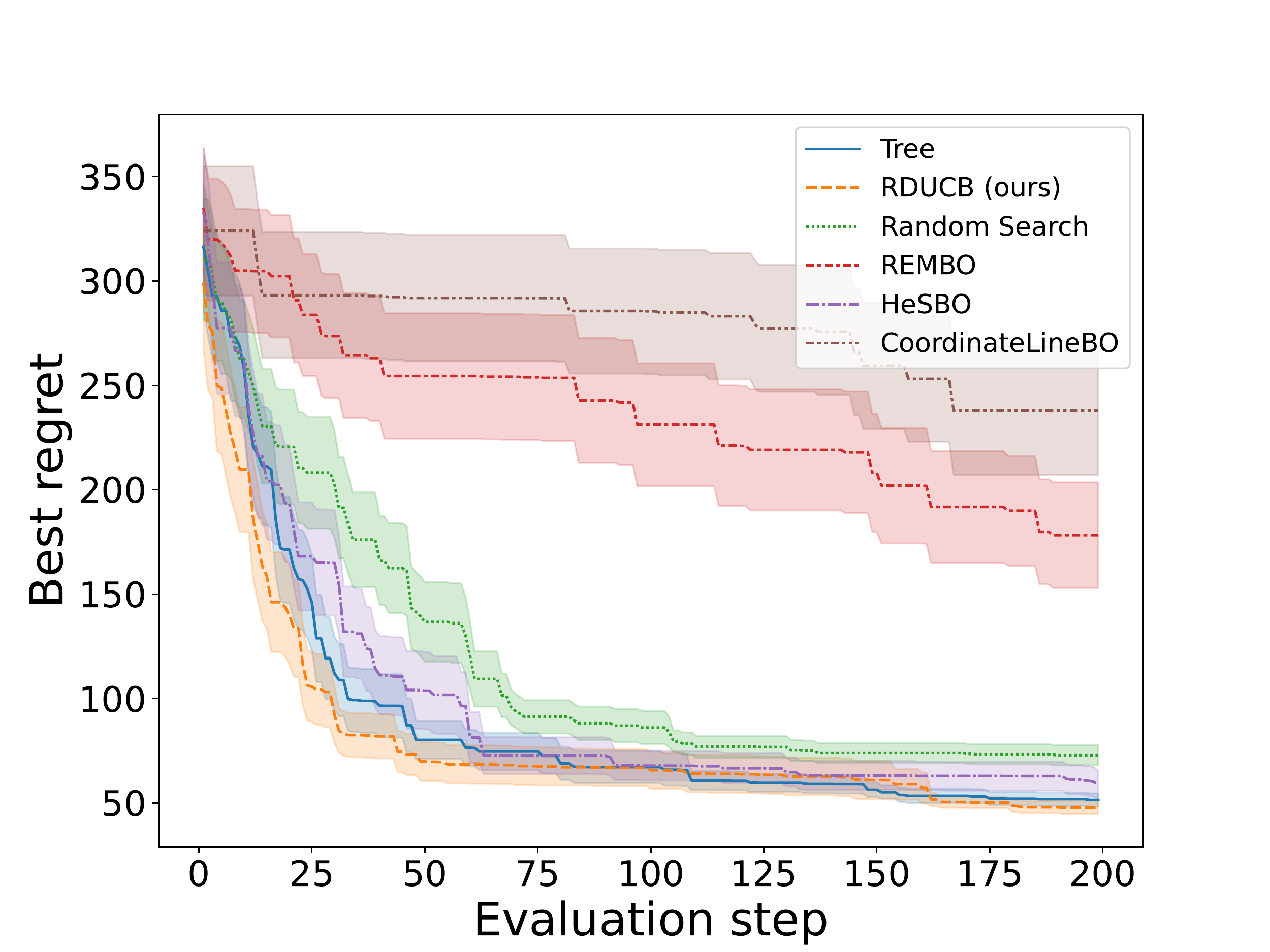}
 \caption{74-d $\texttt{qiu}$ MIP Task}
\end{subfigure}%
\begin{subfigure}[b]{0.33\textwidth}
  \centering
   \includegraphics[trim={0em 0em 5em 2em}, clip=true,width=\linewidth]{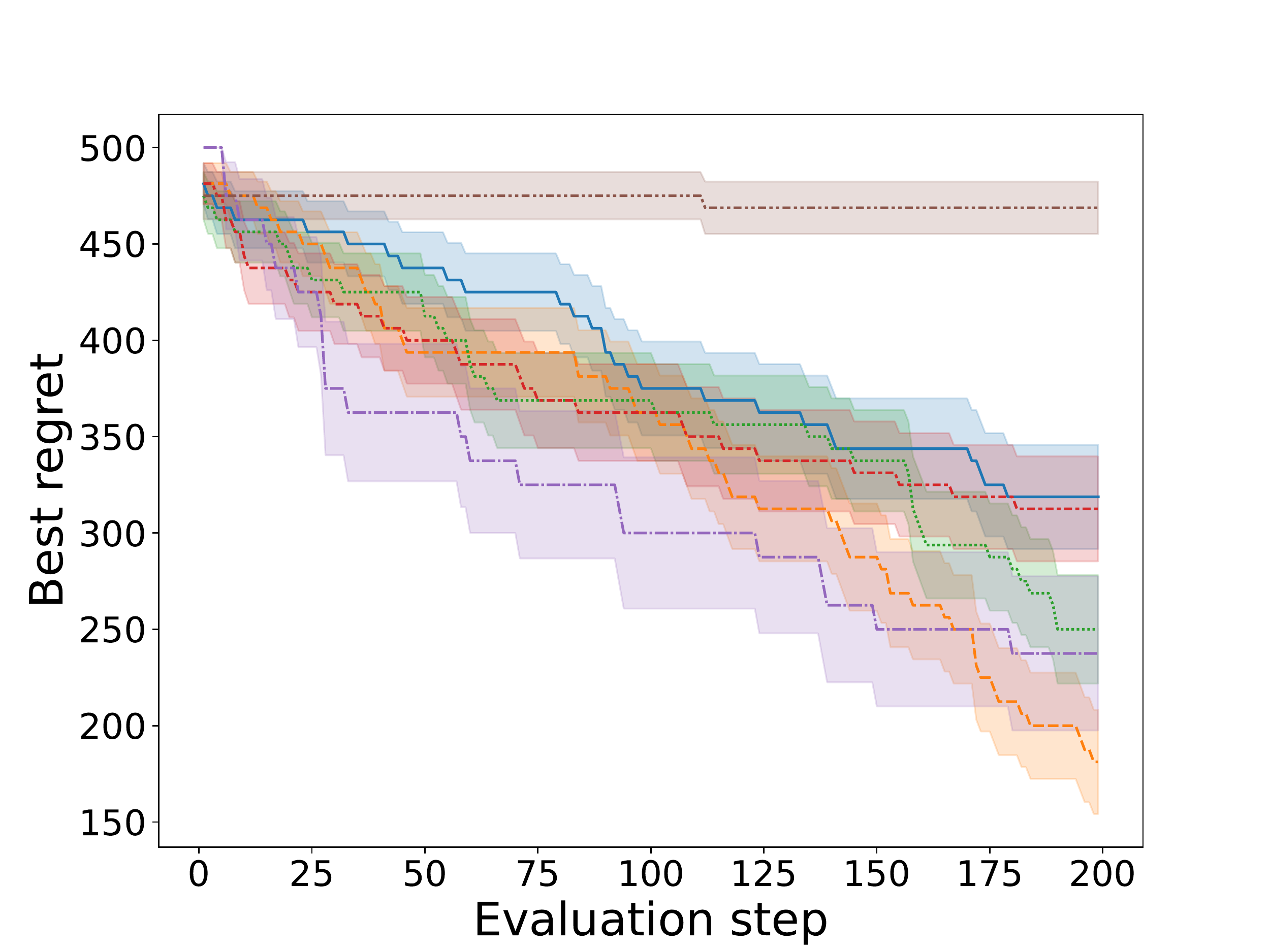}
 \caption{74-d $\texttt{misc05inf}$ MIP Task}
\end{subfigure}%
    \begin{subfigure}[b]{0.33\textwidth}
  \centering
   \includegraphics[trim={0em 0em 5em 2em}, clip=true,width=\linewidth]{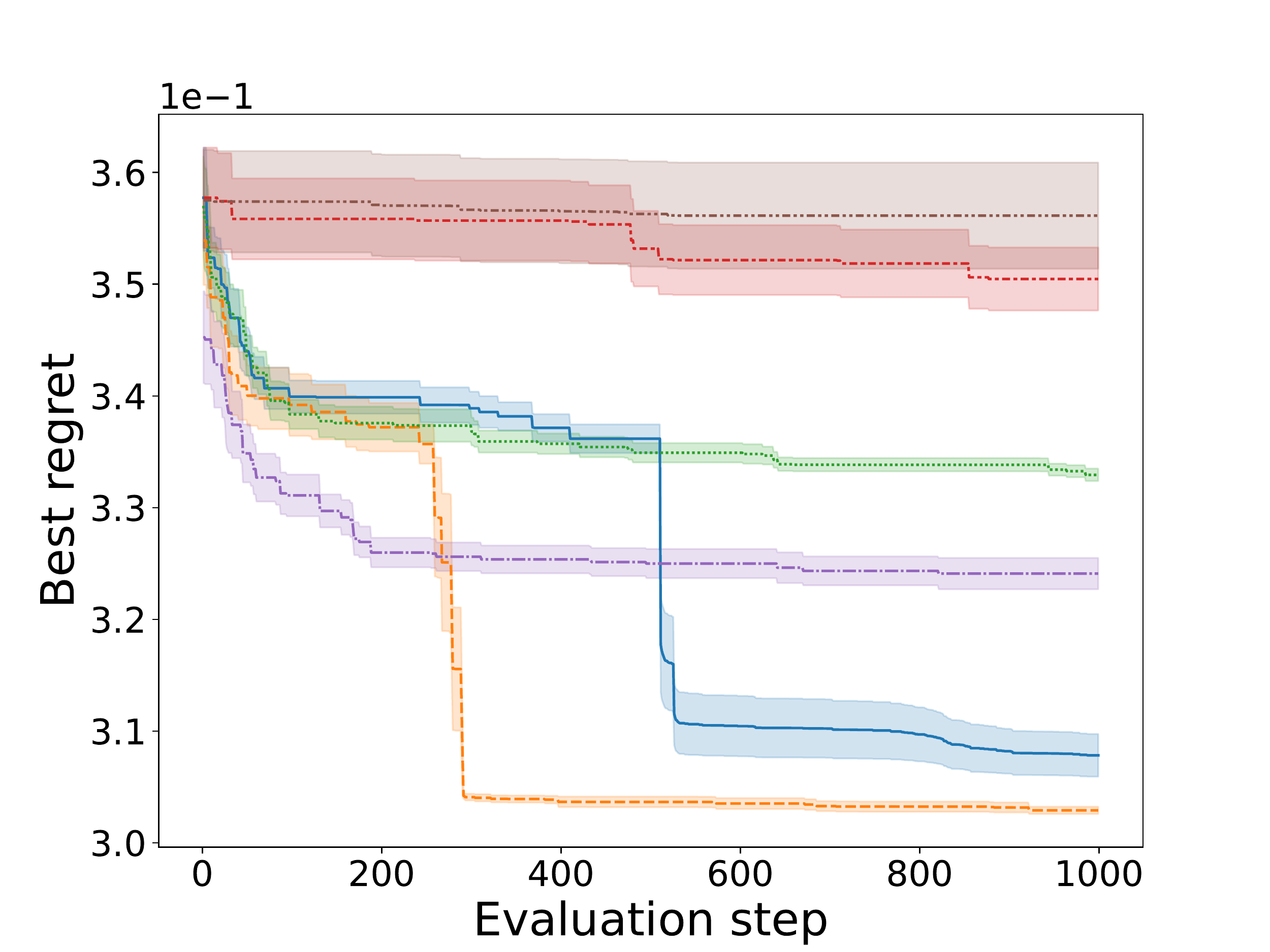}
 \caption{180-d DNA LassoBench Dataset}
\end{subfigure}%
\caption{Best regret results from tuning 74-d MIP problems and the 180-dimensional DNA dataset from the Lasso bench. We average the results over 40 seeds for $\texttt{qiu}$ and 80 seeds for  $\texttt{misc05inf}$. For the LassoBench task, we average over ten random seeds. Tree again refers to the previous state-of-the-art from \cite{han2021high}. We see that RDUCB outperforms others in best regret. Interestingly, the performance gap to other methods increases as the dimensions increase. }
 \label{fig:qiu_misc_dna}
\end{figure*}

\section{Empirical Evaluation}
This section presents experimental results on multiple benchmarks. We compare RDUCB against well-established techniques, including tree-based learners \cite{han2021high}, random embedding BO \cite{wang2013rembo}, Hashing-enhanced
Subspace Bayesian Optimization (HeSBO) \cite{nayebi2019framework}, Coordinate Line BO \cite{kirschner2019adaptive}  and random search\footnote{It is worth noting that REMBO, HeSBO and CoordinateLineBO are presented in all experiments except those involving neural architecture search (NAS). The reason is that NAS baselines operate in discrete search spaces for which these algorithms are not well-suited.}. The comparison against tree-learning-based BO allows us to gauge improvements compared to the prior state-of-the-art. In contrast, the comparison to REMBO, HeSBO and LineBO  sheds light on the advantages we gain when operating in the original  high-dimensional space. 

\newpage

In Appendix \ref{ap:algo_setting}, we provide all algorithm settings used in our experiments. We have open-sourced our code\footnote{\url{https://github.com/huawei-noah/HEBO/tree/master/RDUCB}} to ease the reproducibility of our results. 
\subsection{Benchmark Datasets}

\textbf{Synthethic Functions:} 
We test our method on 20-dimensional Rosebrock, 20-dimensional Hartmann, and 250-dimensional Styblinski-Tang (Stybtang) functions. We calculate the regret as the difference between queried function value and the theoretical minimum. 

\textbf{Neural Network Hyperparameter Tuning:}
In the second set of experiments, we consider the neural architecture search (NAS) hyperparameter tuning benchmark \cite{zela2020bench}. This benchmark consists of precomputed validation mean-squared errors for different combinations of hyperparameters in a two-layer feed-forward neural network trained on four different datasets. This experiment considers \emph{mixed} search spaces, e.g., learning and drop-out rates (continuous), the sizes of hidden layers and activation types (discrete).

\textbf{Mixed Integer Programming:} 
We consider the problem of tuning heuristic hyperparameters for the mixed integer programming (MIP) solver LPSolve \cite{berkelaar2015package}. This domain is high-dimensional with a 74-dimensional search space. We consider three MIP problems varying in difficulty with the $\texttt{misc05inf}$ instance being the hardest. For each problem, the regret value is the duality gap of the best solution found after 5 seconds with given hyperparameters. We cap the maximum instantaneous regret at 500, and running the solver once is considered as one query. 

\textbf{Weighted Lasso Tuning:} 
The last problem we consider is tuning the LassoBench \cite{vsehic2022lassobench}. LassoBench is a set of benchmark problems where BO tunes the weights for a weighted Lasso model. Hence, the number of dimensions of the search space scales with the number of features in the dataset. The highest dimension we study in this experiment is the 180 DNA data.

\subsection{Regret Results}
We run all algorithms varying the number of seeds per-each benchmark to allow for statistically significant results on the tasks introduced above. We report some of these results in this section (Figures \ref{fig:stybtang_protein_parkinson} and \ref{fig:qiu_misc_dna}) and defer others to Appendix \ref{App:Exps} due to space constraints. Every algorithm started with 10 initial randomly sampled points. Note that for HeSBO, the initial points were sampled in the algorithm's lower-dimensional space (hence the initial regret is sometimes different from other algorithms). 

In general, we notice that RDUCB outperforms other baselines. We see that our improvements are amplified in high dimensional settings, as reported in Figures \ref{fig:stybtang_protein_parkinson}(a) and Figures \ref{fig:qiu_misc_dna}(a),  \ref{fig:qiu_misc_dna}(b) and  \ref{fig:qiu_misc_dna}(c). In three experiments, RDUCB draws in terms of final regret with Tree, and in two benchmarks RDUCB (minorly) underperforms compared to Tree (see Appendix \ref{App:Exps}). We note that the cases when RDUCB draws or underperforms mostly correspond to low-dimensional settings. We expand on this in the next section.

\subsection{Scalability as Dimensionalities Increase}
We conduct an ablation study to test how our method scales with the number of dimensions and to understand its potential limitations. We choose one synthetic function and one real-world problem. For the synthetic function, we vary the dimensionality of the Stybtang function, while in the real-world case, we tackle the LassoBench DNA problem. For the latter, we produce different versions of this problem with varying dimensionalities. To achieve this, we turn some dimensions off by permanently setting them to zero and allowing the algorithm to vary only the remaining ones.
 
 We compare the best regret achieved by the algorithm in \cite{han2021high} (titled Tree in the figure) and RDUCB after 500 iterations for Stybtang and 700 iterations for LassoBench DNA, respectively. We show the results of ablation in Figure \ref{fig:ablation_bar}. Although for a small number of dimensions, Tree can slightly outperform RDUCB, we see that as the number of dimensions increases, RDUCB starts to outperform, and the gap widens for higher-dimensional problems.

 \begin{figure}[h]
    \centering
    \includegraphics[trim={4em 0.5em 5em 2em}, clip=true, width=\columnwidth]{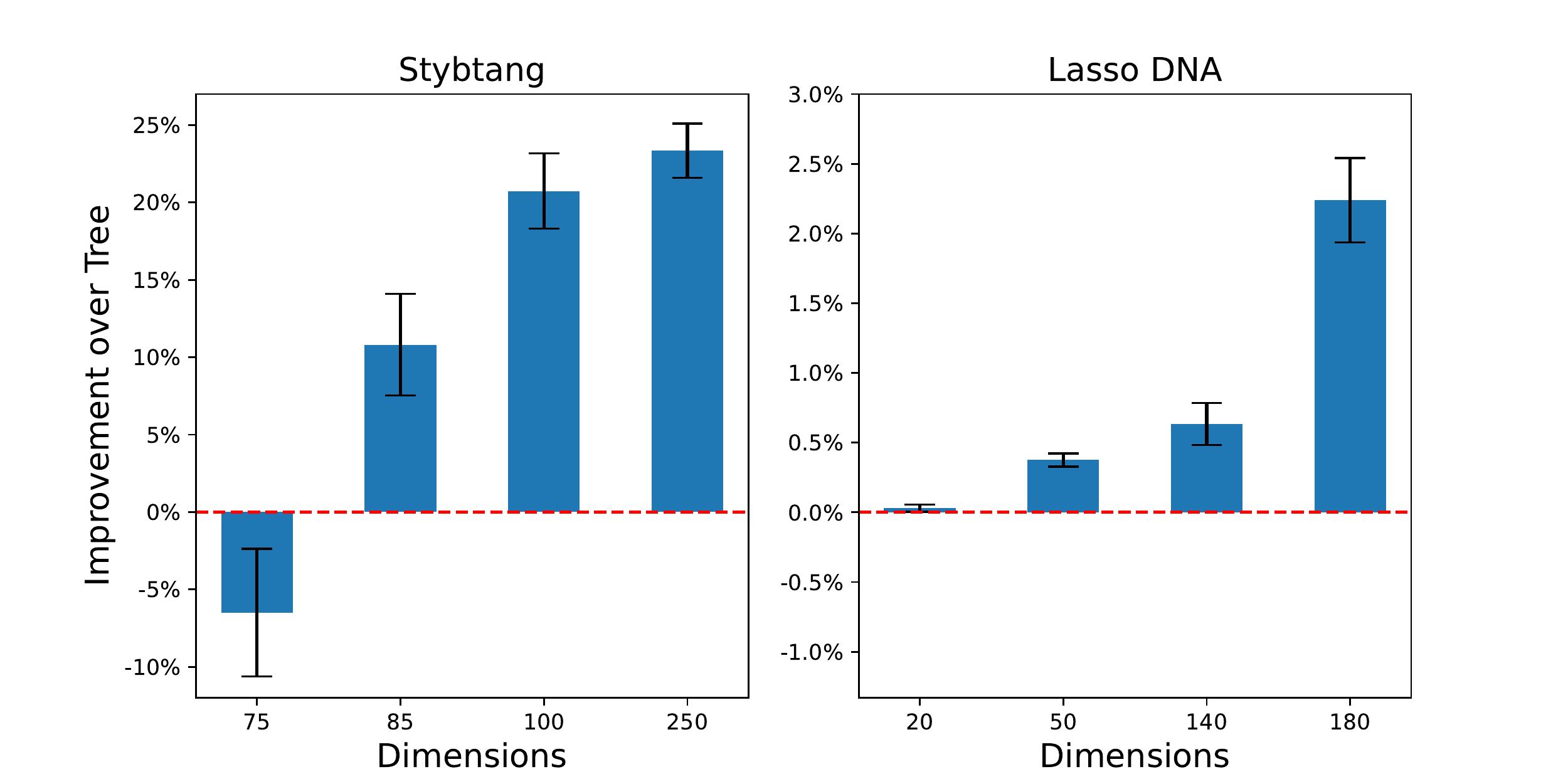}
    \caption{Improvement of RDUCB over Tree \cite{han2021high} in terms of final best regret for different dimensionality of the problems. Error bars correspond to standard errors.}
    \label{fig:ablation_bar}
\end{figure}

\subsection{Plug-and-Play for RDHEBO}\label{Sec:RD-HEBO}
RDUCB presents a simple  modification strategy to modelling using GPs, allowing it to plug on top of existing BO algorithms. This section supports our claim by introducing RDHEBO, an adaptation of HEBO \cite{cowen2022hebo} that uses our random decomposition scheme. We keep all the other components (warping functions and multi-objective acquisitions) of HEBO unchanged and present a comparison conducted on the highest dimensional Bayesmark tasks in Table 1. We can see that our modification has improved the performance of HEBO across those tasks, rising from a score of 92.68 to 93.67 in the MLP-Adam task, for instance.  

\begin{table}[h]
    \centering
    \begin{tabular}{c|c|c}
        Problem &\textbf{HEBO} & \textbf{RDHEBO} \\
        \midrule
        MLP-Adam &  $92.68 \pm 0.22$ & \bm{$93.67 \pm 0.30$} \\
        MLP-SGD &  $90.66 \pm 0.81$ & \bm{$91.65 \pm 0.10$}\\
        DT & $79.42 \pm 0.45$ & \bm{$80.79 \pm 0.15$}\\
        RF & $84.97 \pm 0.32$ & \bm{$87.64 \pm 2.00$} \\
        \midrule
        Average & $86.93 \pm 0.45$ & $\textbf{88.44} \pm \textbf{0.64}$
    \end{tabular}
    \caption{Bayesmark normalised score with standard errors for HEBO and RDHEBO demonstrating that our modifications can improve HEBO's performance. We run 10 seeds for MLP-Adam and DT and 15 for MLP-SGD and RF.}
    \label{tab:my_label}
\end{table}

\section{Related Work} \label{Sec:RelatedWork}

\textbf{Random projection methods} project the inputs from a high dimensional space to a space of lower dimensionality by randomised mappings. 
To do so, REMBO \cite{wang2013rembo} uses a matrix with entries sampled from a normal distribution. Though successful, REMBO tends to over-explore regions that are on the boundary of the  domain of the black-box function. The authors in \cite{letham2020re} proposed ALEBO, an algorithm that addresses some of REMBO's problems, while \citet{nayebi2019framework} use different projection types altogether. The algorithm they propose - HeSBO, initialises a randomised hashing function that maps back to the original space, leading to improved results. 
As opposed to all those techniques, RDUCB takes a different direction of decomposing the original space rather than projecting to lower dimensional manifolds.  


\textbf{LineBO methods} \cite{kirschner2019adaptive} conduct optimisation over an affine, one-dimensional subspace. This subspace is chosen so that the best point found so far belongs to it. There are different variants of LineBO, where the direction of this one-dimensional space is chosen randomly (Random LineBO), aligned with a randomly chosen dimension (Coordinate LineBO) or along a gradient estimate (Descent LineBO). We chose to compare RDUCB with Coordinate LineBO, as in most cases it outperforms other LineBO variants.

\textbf{Latent space methods} perform the optimisation in some latent space and then utilise a generative model as a projection to original space. A commonly taken approach is to utilise the latent space of a  Variational AutoEncoder \cite{kingma2013auto}, pretrained on some larger batch of existing data. However, over the course of this pretraining, we need to ensure that the obtained latent space is well-suited for BO. There have been numerous VAE-based algorithms proposed \cite{eissman2018bayesian, gomez2018automatic, zhang2019high, tripp2020sample,  griffiths2020constrained, siivola2021good, grosnit2021high}, employing various mechanisms to achieve this goal.

While using deep networks to scale BO is prominent and essential, RDUCB constitutes an orthogonal research direction where our discoveries, at least in the modelling component, can (easily) plug-and-play in latent spaces, e.g., when using deep kernel GPs \cite{wilson2016deep} in general. We instantiated one such plug-and-play use-case in this paper (see Section \ref{Sec:RD-HEBO}) and planned to further investigate such applications in the future. 

\textbf{Dropout methods} \cite{li2017dropout} randomly drop several dimensions and optimise only a subset at each step. When querying for new values of the black-box function, the value of dimensions that were not optimised at a given iteration are selected according to some filling strategy that can either copy the value of the best point so far or choose a random value. Though successful in isolated instances, dropout methods exhibit high-sample complexity.  

\textbf{SAASBO} \cite{eriksson2021high} adopts a hierarchical Bayesian model, putting a prior on the length scale with a high probability mass concentrated at zero. Hence, a dimension will be effectively removed from the model unless there is enough evidence that it dramatically impacts the black-box function value.

\textbf{Trust region methods} tackle high dimensional spaces by optimising locally within a chosen trust region. The TuRBO algorithm \cite{eriksson2019scalable} maintains several trust regions and selects one at a given time using a multi-armed bandit strategy. CASMOPOLITAN \cite{wan2021think} extends this idea to handle categorical and mixed input spaces. 

Our research in this paper is orthogonal to SAASBO and trust region methods. In the future, we plan to integrate our modelling component with the works in \cite{eriksson2019scalable} and \cite{eriksson2021high}.

\section{Discussion \& Future Work}
This paper introduced random decompositions to scale BO to high-dimensional spaces. Our algorithm, RDUCB, is simple to implement, empirically effective and theoretically grounded. We compared our method against the prior decomposition state-of-the-art technique and showed improved best-regret results on a wide set of publicly available benchmarks. While this is the first step in understanding data-independent rules for high dimensional BO, there is a number of challenges that we now elaborate on. 

Currently, RDUCB is not capable of handling non-numerical inputs (e.g., graphs or strings). This problem is not only tied to RDUCB but also challenges other decomposition-based techniques since we cannot easily define a notion of decomposition for such input types. We plan to investigate such decomposition rules in the future.  

Apart from tackling technical challenges, in the future, we plan to scale RDUCB to thousands of dimensions. We hope to benefit from distributed and high-performance computing. Of course, this direction requires us to devise novel consensus-based acquisition optimisers, which have recently seen progress in non-convex settings \cite{zhang2022distributed}.

Although we applied the modelling component of RDUCB to HEBO, we believe many BO frameworks and algorithms can benefit from our decomposition scheme. Those include trust-region methods, latent space techniques, and combinatorial BO solvers. In the future, we will also investigate this direction in combination with scalability.

\bibliography{example_paper}

\begin{thebibliography}{53}
\providecommand{\natexlab}[1]{#1}
\providecommand{\url}[1]{\texttt{#1}}
\expandafter\ifx\csname urlstyle\endcsname\relax
  \providecommand{\doi}[1]{doi: #1}\else
  \providecommand{\doi}{doi: \begingroup \urlstyle{rm}\Url}\fi

\bibitem[Asuncion \& Newman(2007)Asuncion and Newman]{asuncion2007uci}
Asuncion, A. and Newman, D.
\newblock Uci machine learning repository, 2007.

\bibitem[Berkelaar et~al.(2015)]{berkelaar2015package}
Berkelaar, M. et~al.
\newblock Package ‘lpsolve’, 2015.

\bibitem[Berkenkamp et~al.(2019)Berkenkamp, Schoellig, and Krause]{Berkenkamp}
Berkenkamp, F., Schoellig, A.~P., and Krause, A.
\newblock No-regret bayesian optimization with unknown hyperparameters.
\newblock \emph{arXiv preprint arXiv:1901.03357}, 2019.

\bibitem[Bogunovic \& Krause(2021)Bogunovic and
  Krause]{bogunovic2021misspecified}
Bogunovic, I. and Krause, A.
\newblock Misspecified gaussian process bandit optimization.
\newblock \emph{Advances in Neural Information Processing Systems 34}, 2021.

\bibitem[Chowdhury \& Gopalan(2017)Chowdhury and
  Gopalan]{chowdhury2017kernelized}
Chowdhury, S.~R. and Gopalan, A.
\newblock On kernelized multi-armed bandits.
\newblock In \emph{International Conference on Machine Learning}, pp.\
  844--853. PMLR, 2017.

\bibitem[Cowen-Rivers et~al.(2022)Cowen-Rivers, Lyu, Tutunov, Wang, Grosnit,
  Griffiths, Maraval, Jianye, Wang, Peters, et~al.]{cowen2022hebo}
Cowen-Rivers, A.~I., Lyu, W., Tutunov, R., Wang, Z., Grosnit, A., Griffiths,
  R.~R., Maraval, A.~M., Jianye, H., Wang, J., Peters, J., et~al.
\newblock Hebo: pushing the limits of sample-efficient hyper-parameter
  optimisation.
\newblock \emph{Journal of Artificial Intelligence Research}, 74:\penalty0
  1269--1349, 2022.

\bibitem[Durrande et~al.(2012)Durrande, Ginsbourger, and
  Roustant]{durrande2012additive}
Durrande, N., Ginsbourger, D., and Roustant, O.
\newblock Additive covariance kernels for high-dimensional gaussian process
  modeling.
\newblock In \emph{Annales de la Facult{\'e} des sciences de Toulouse:
  Math{\'e}matiques}, volume~21, pp.\  481--499, 2012.

\bibitem[Duvenaud et~al.(2011)Duvenaud, Nickisch, and
  Rasmussen]{duvenaud2011additive}
Duvenaud, D.~K., Nickisch, H., and Rasmussen, C.
\newblock Additive gaussian processes.
\newblock \emph{Advances in neural information processing systems}, 24, 2011.

\bibitem[Eissman et~al.(2018)Eissman, Levy, Shu, Bartzsch, and
  Ermon]{eissman2018bayesian}
Eissman, S., Levy, D., Shu, R., Bartzsch, S., and Ermon, S.
\newblock Bayesian optimization and attribute adjustment.
\newblock In \emph{Proc. 34th Conference on Uncertainty in Artificial
  Intelligence}, 2018.

\bibitem[Eriksson \& Jankowiak(2021)Eriksson and Jankowiak]{eriksson2021high}
Eriksson, D. and Jankowiak, M.
\newblock High-dimensional bayesian optimization with sparse axis-aligned
  subspaces.
\newblock In \emph{Uncertainty in Artificial Intelligence}, pp.\  493--503.
  PMLR, 2021.

\bibitem[Eriksson et~al.(2019)Eriksson, Pearce, Gardner, Turner, and
  Poloczek]{eriksson2019scalable}
Eriksson, D., Pearce, M., Gardner, J., Turner, R.~D., and Poloczek, M.
\newblock Scalable global optimization via local bayesian optimization.
\newblock \emph{Advances in neural information processing systems}, 32, 2019.

\bibitem[Frazier et~al.(2008)Frazier, Powell, and
  Dayanik]{frazier2008knowledge}
Frazier, P.~I., Powell, W.~B., and Dayanik, S.
\newblock A knowledge-gradient policy for sequential information collection.
\newblock \emph{SIAM Journal on Control and Optimization}, 47\penalty0
  (5):\penalty0 2410--2439, 2008.

\bibitem[G{\'o}mez-Bombarelli et~al.(2018)G{\'o}mez-Bombarelli, Wei, Duvenaud,
  Hern{\'a}ndez-Lobato, S{\'a}nchez-Lengeling, Sheberla, Aguilera-Iparraguirre,
  Hirzel, Adams, and Aspuru-Guzik]{gomez2018automatic}
G{\'o}mez-Bombarelli, R., Wei, J.~N., Duvenaud, D., Hern{\'a}ndez-Lobato,
  J.~M., S{\'a}nchez-Lengeling, B., Sheberla, D., Aguilera-Iparraguirre, J.,
  Hirzel, T.~D., Adams, R.~P., and Aspuru-Guzik, A.
\newblock Automatic chemical design using a data-driven continuous
  representation of molecules.
\newblock \emph{ACS central science}, 4\penalty0 (2):\penalty0 268--276, 2018.

\bibitem[Griffiths \& Hern{\'a}ndez-Lobato(2020)Griffiths and
  Hern{\'a}ndez-Lobato]{griffiths2020constrained}
Griffiths, R.-R. and Hern{\'a}ndez-Lobato, J.~M.
\newblock Constrained bayesian optimization for automatic chemical design using
  variational autoencoders.
\newblock \emph{Chemical science}, 11\penalty0 (2):\penalty0 577--586, 2020.

\bibitem[Grosnit et~al.(2021{\natexlab{a}})Grosnit, Cowen-Rivers, Tutunov,
  Griffiths, Wang, and Bou-Ammar]{grosnit2021we}
Grosnit, A., Cowen-Rivers, A.~I., Tutunov, R., Griffiths, R.-R., Wang, J., and
  Bou-Ammar, H.
\newblock Are we forgetting about compositional optimisers in bayesian
  optimisation?
\newblock \emph{The Journal of Machine Learning Research}, 22\penalty0
  (1):\penalty0 7183--7260, 2021{\natexlab{a}}.

\bibitem[Grosnit et~al.(2021{\natexlab{b}})Grosnit, Tutunov, Maraval,
  Griffiths, Cowen-Rivers, Yang, Zhu, Lyu, Chen, Wang, et~al.]{grosnit2021high}
Grosnit, A., Tutunov, R., Maraval, A.~M., Griffiths, R.-R., Cowen-Rivers,
  A.~I., Yang, L., Zhu, L., Lyu, W., Chen, Z., Wang, J., et~al.
\newblock High-dimensional bayesian optimisation with variational autoencoders
  and deep metric learning.
\newblock \emph{arXiv preprint arXiv:2106.03609}, 2021{\natexlab{b}}.

\bibitem[Grosnit et~al.(2022)Grosnit, Malherbe, Tutunov, Wan, Wang, and
  Ammar]{grosnit2022boils}
Grosnit, A., Malherbe, C., Tutunov, R., Wan, X., Wang, J., and Ammar, H.~B.
\newblock Boils: bayesian optimisation for logic synthesis.
\newblock In \emph{2022 Design, Automation \& Test in Europe Conference \&
  Exhibition (DATE)}, pp.\  1193--1196. IEEE, 2022.

\bibitem[Han et~al.(2021)Han, Arora, and Scarlett]{han2021high}
Han, E., Arora, I., and Scarlett, J.
\newblock High-dimensional bayesian optimization via tree-structured additive
  models.
\newblock In \emph{Proceedings of the AAAI Conference on Artificial
  Intelligence}, volume~35, pp.\  7630--7638, 2021.

\bibitem[Kandasamy et~al.(2015)Kandasamy, Schneider, and
  P{\'o}czos]{kandasamy2015high}
Kandasamy, K., Schneider, J., and P{\'o}czos, B.
\newblock High dimensional bayesian optimisation and bandits via additive
  models.
\newblock In \emph{International conference on machine learning}, pp.\
  295--304. PMLR, 2015.

\bibitem[Kandasamy et~al.(2017)Kandasamy, Dasarathy, Schneider, and
  P{\'o}czos]{kandasamy2017multi}
Kandasamy, K., Dasarathy, G., Schneider, J., and P{\'o}czos, B.
\newblock Multi-fidelity bayesian optimisation with continuous approximations.
\newblock In \emph{International Conference on Machine Learning}, pp.\
  1799--1808. PMLR, 2017.

\bibitem[Kandasamy et~al.(2018)Kandasamy, Krishnamurthy, Schneider, and
  P{\'o}czos]{kandasamy2018parallelised}
Kandasamy, K., Krishnamurthy, A., Schneider, J., and P{\'o}czos, B.
\newblock Parallelised bayesian optimisation via thompson sampling.
\newblock In \emph{International Conference on Artificial Intelligence and
  Statistics}, pp.\  133--142. PMLR, 2018.

\bibitem[Kingma \& Welling(2013)Kingma and Welling]{kingma2013auto}
Kingma, D.~P. and Welling, M.
\newblock Auto-encoding variational bayes.
\newblock \emph{arXiv preprint arXiv:1312.6114}, 2013.

\bibitem[Kirschner et~al.(2019)Kirschner, Mutny, Hiller, Ischebeck, and
  Krause]{kirschner2019adaptive}
Kirschner, J., Mutny, M., Hiller, N., Ischebeck, R., and Krause, A.
\newblock Adaptive and safe bayesian optimization in high dimensions via
  one-dimensional subspaces.
\newblock In \emph{International Conference on Machine Learning}, pp.\
  3429--3438. PMLR, 2019.

\bibitem[Kirschner et~al.(2020)Kirschner, Bogunovic, Jegelka, and
  Krause]{kirschner2020distributionally}
Kirschner, J., Bogunovic, I., Jegelka, S., and Krause, A.
\newblock Distributionally robust bayesian optimization.
\newblock In \emph{International Conference on Artificial Intelligence and
  Statistics}, pp.\  2174--2184. PMLR, 2020.

\bibitem[Letham et~al.(2020)Letham, Calandra, Rai, and Bakshy]{letham2020re}
Letham, B., Calandra, R., Rai, A., and Bakshy, E.
\newblock Re-examining linear embeddings for high-dimensional bayesian
  optimization.
\newblock \emph{Advances in neural information processing systems},
  33:\penalty0 1546--1558, 2020.

\bibitem[Li et~al.(2017)Li, Gupta, Rana, Nguyen, Venkatesh, and
  Shilton]{li2017dropout}
Li, C., Gupta, S., Rana, S., Nguyen, V., Venkatesh, S., and Shilton, A.
\newblock High dimensional bayesian optimization using dropout.
\newblock In \emph{Proceedings of the 26th International Joint Conference on
  Artificial Intelligence}, pp.\  2096--2102, 2017.

\bibitem[Li et~al.(2018)Li, Gupta, Rana, Nguyen, Venkatesh, and
  Shilton]{li2018high}
Li, C., Gupta, S., Rana, S., Nguyen, V., Venkatesh, S., and Shilton, A.
\newblock High dimensional bayesian optimization using dropout.
\newblock \emph{arXiv preprint arXiv:1802.05400}, 2018.

\bibitem[Lu et~al.(2022)Lu, Boukouvalas, and Hensman]{lu2022additive}
Lu, X., Boukouvalas, A., and Hensman, J.
\newblock Additive gaussian processes revisited.
\newblock In \emph{International Conference on Machine Learning}, pp.\
  14358--14383. PMLR, 2022.

\bibitem[Marchant \& Ramos(2012)Marchant and Ramos]{marchant2012bayesian}
Marchant, R. and Ramos, F.
\newblock Bayesian optimisation for intelligent environmental monitoring.
\newblock In \emph{2012 IEEE/RSJ international conference on intelligent robots
  and systems}, pp.\  2242--2249. IEEE, 2012.

\bibitem[Moriconi et~al.(2020)Moriconi, Deisenroth, and
  Sesh~Kumar]{moriconi2020high}
Moriconi, R., Deisenroth, M.~P., and Sesh~Kumar, K.
\newblock High-dimensional bayesian optimization using low-dimensional feature
  spaces.
\newblock \emph{Machine Learning}, 109\penalty0 (9):\penalty0 1925--1943, 2020.

\bibitem[Nayebi et~al.(2019)Nayebi, Munteanu, and
  Poloczek]{nayebi2019framework}
Nayebi, A., Munteanu, A., and Poloczek, M.
\newblock A framework for bayesian optimization in embedded subspaces.
\newblock In \emph{International Conference on Machine Learning}, pp.\
  4752--4761. PMLR, 2019.

\bibitem[Nguyen et~al.(2017)Nguyen, Gupta, Rana, Li, and
  Venkatesh]{nguyen2017regret}
Nguyen, V., Gupta, S., Rana, S., Li, C., and Venkatesh, S.
\newblock Regret for expected improvement over the best-observed value and
  stopping condition.
\newblock In \emph{Asian conference on machine learning}, pp.\  279--294. PMLR,
  2017.

\bibitem[Qamar \& Tokdar(2014)Qamar and Tokdar]{qamar2014additive}
Qamar, S. and Tokdar, S.~T.
\newblock Additive gaussian process regression.
\newblock \emph{arXiv preprint arXiv:1411.7009}, 2014.

\bibitem[Rana et~al.(2017)Rana, Li, Gupta, Nguyen, and Venkatesh]{rana2017high}
Rana, S., Li, C., Gupta, S., Nguyen, V., and Venkatesh, S.
\newblock High dimensional bayesian optimization with elastic gaussian process.
\newblock In \emph{International conference on machine learning}, pp.\
  2883--2891. PMLR, 2017.

\bibitem[Rolland et~al.(2018)Rolland, Scarlett, Bogunovic, and
  Cevher]{rolland2018high}
Rolland, P., Scarlett, J., Bogunovic, I., and Cevher, V.
\newblock High-dimensional bayesian optimization via additive models with
  overlapping groups.
\newblock In \emph{International conference on artificial intelligence and
  statistics}, pp.\  298--307. PMLR, 2018.

\bibitem[{\v{S}}ehi{\'c} et~al.(2022){\v{S}}ehi{\'c}, Gramfort, Salmon, and
  Nardi]{vsehic2022lassobench}
{\v{S}}ehi{\'c}, K., Gramfort, A., Salmon, J., and Nardi, L.
\newblock Lassobench: A high-dimensional hyperparameter optimization benchmark
  suite for lasso.
\newblock In \emph{International Conference on Automated Machine Learning},
  pp.\  2--1. PMLR, 2022.

\bibitem[Shahriari et~al.(2015)Shahriari, Swersky, Wang, Adams, and
  De~Freitas]{shahriari2015taking}
Shahriari, B., Swersky, K., Wang, Z., Adams, R.~P., and De~Freitas, N.
\newblock Taking the human out of the loop: A review of bayesian optimization.
\newblock \emph{Proceedings of the IEEE}, 104\penalty0 (1):\penalty0 148--175,
  2015.

\bibitem[Siivola et~al.(2021)Siivola, Paleyes, Gonz{\'a}lez, and
  Vehtari]{siivola2021good}
Siivola, E., Paleyes, A., Gonz{\'a}lez, J., and Vehtari, A.
\newblock Good practices for bayesian optimization of high dimensional
  structured spaces.
\newblock \emph{Applied AI Letters}, 2\penalty0 (2):\penalty0 e24, 2021.

\bibitem[Srinivas et~al.(2009)Srinivas, Krause, Kakade, and
  Seeger]{srinivas2009gaussian}
Srinivas, N., Krause, A., Kakade, S.~M., and Seeger, M.
\newblock Gaussian process optimization in the bandit setting: No regret and
  experimental design.
\newblock \emph{arXiv preprint arXiv:0912.3995}, 2009.

\bibitem[Srinivas et~al.(2012)Srinivas, Krause, Kakade, and
  Seeger]{srinivas2012information}
Srinivas, N., Krause, A., Kakade, S.~M., and Seeger, M.~W.
\newblock Information-theoretic regret bounds for gaussian process optimization
  in the bandit setting.
\newblock \emph{IEEE transactions on information theory}, 58\penalty0
  (5):\penalty0 3250--3265, 2012.

\bibitem[Tripp et~al.(2020)Tripp, Daxberger, and
  Hern{\'a}ndez-Lobato]{tripp2020sample}
Tripp, A., Daxberger, E., and Hern{\'a}ndez-Lobato, J.~M.
\newblock Sample-efficient optimization in the latent space of deep generative
  models via weighted retraining.
\newblock \emph{Advances in Neural Information Processing Systems},
  33:\penalty0 11259--11272, 2020.

\bibitem[Turner et~al.(2021)Turner, Eriksson, McCourt, Kiili, Laaksonen, Xu,
  and Guyon]{turner2021bayesian}
Turner, R., Eriksson, D., McCourt, M., Kiili, J., Laaksonen, E., Xu, Z., and
  Guyon, I.
\newblock Bayesian optimization is superior to random search for machine
  learning hyperparameter tuning: Analysis of the black-box optimization
  challenge 2020.
\newblock In \emph{NeurIPS 2020 Competition and Demonstration Track}, pp.\
  3--26. PMLR, 2021.

\bibitem[Wan et~al.(2021)Wan, Nguyen, Ha, Ru, Lu, and Osborne]{wan2021think}
Wan, X., Nguyen, V., Ha, H., Ru, B., Lu, C., and Osborne, M.~A.
\newblock Think global and act local: Bayesian optimisation over
  high-dimensional categorical and mixed search spaces.
\newblock \emph{arXiv preprint arXiv:2102.07188}, 2021.

\bibitem[Wang et~al.(2013)Wang, Zoghi, Hutter, Matheson, De~Freitas,
  et~al.]{wang2013rembo}
Wang, Z., Zoghi, M., Hutter, F., Matheson, D., De~Freitas, N., et~al.
\newblock Bayesian optimization in high dimensions via random embeddings.
\newblock In \emph{IJCAI}, volume~13, pp.\  1778--1784, 2013.

\bibitem[Wang et~al.(2016)Wang, Hutter, Zoghi, Matheson, and
  De~Feitas]{wang2016bayesian}
Wang, Z., Hutter, F., Zoghi, M., Matheson, D., and De~Feitas, N.
\newblock Bayesian optimization in a billion dimensions via random embeddings.
\newblock \emph{Journal of Artificial Intelligence Research}, 55:\penalty0
  361--387, 2016.

\bibitem[Williams \& Rasmussen(2006)Williams and
  Rasmussen]{williams2006gaussian}
Williams, C.~K. and Rasmussen, C.~E.
\newblock \emph{Gaussian processes for machine learning}, volume~2.
\newblock MIT press Cambridge, MA, 2006.

\bibitem[Wilson et~al.(2016)Wilson, Hu, Salakhutdinov, and
  Xing]{wilson2016deep}
Wilson, A.~G., Hu, Z., Salakhutdinov, R., and Xing, E.~P.
\newblock Deep kernel learning.
\newblock In \emph{Artificial intelligence and statistics}, pp.\  370--378.
  PMLR, 2016.

\bibitem[Wilson et~al.(2017)Wilson, Moriconi, Hutter, and
  Deisenroth]{wilson2017reparameterization}
Wilson, J.~T., Moriconi, R., Hutter, F., and Deisenroth, M.~P.
\newblock The reparameterization trick for acquisition functions.
\newblock \emph{arXiv preprint arXiv:1712.00424}, 2017.

\bibitem[Zela et~al.(2020)Zela, Siems, and Hutter]{zela2020bench}
Zela, A., Siems, J., and Hutter, F.
\newblock Nas-bench-1shot1: Benchmarking and dissecting one-shot neural
  architecture search.
\newblock \emph{arXiv preprint arXiv:2001.10422}, 2020.

\bibitem[Zhang et~al.(2019)Zhang, Li, and Su]{zhang2019high}
Zhang, M., Li, H., and Su, S.
\newblock High dimensional bayesian optimization via supervised dimension
  reduction.
\newblock \emph{arXiv preprint arXiv:1907.08953}, 2019.

\bibitem[Zhang et~al.(2021)Zhang, Zhang, and Frazier]{zhang2021constrained}
Zhang, Y., Zhang, X., and Frazier, P.
\newblock Constrained two-step look-ahead bayesian optimization.
\newblock \emph{Advances in Neural Information Processing Systems},
  34:\penalty0 12563--12575, 2021.

\bibitem[Zhang et~al.(2022)Zhang, Li, and Wang]{zhang2022distributed}
Zhang, Y., Li, X., and Wang, L.
\newblock Distributed h-infinity consensus of heterogeneous multi-agent systems
  with nonconvex constraints.
\newblock \emph{ISA transactions}, 131:\penalty0 160--166, 2022.

\bibitem[Ziegel(2003)]{ziegel2003elements}
Ziegel, E.~R.
\newblock The elements of statistical learning, 2003.

\end{thebibliography}
\bibliographystyle{icml2023}

\newpage
\appendix

\onecolumn
\section{Proofs of Theoretical Results} \label{ap:proofs}
\subsection{Proof of Theorem \ref{th:regretbound}} \label{ap:proof_regretbound}
\regretbound*
\begin{proof} Without the loss of generality, we assume kernel for every decomposition is normalised, i.e. $\forall_{g \in \mathcal{G}} k^g(\bm{x}, \bm{x}') \le 1$ and that our function observations are corrupted by some $\sigma_n$-subgaussian noise.
Our proof follows the idea of the proof of Theorem 1 from \cite{bogunovic2021misspecified}, with some differences. Let $k_t = k^{g_t}$ be the kernel defined by decomposition $g_t$ selected at time $t$. $\hat{f}^t = \textrm{argmin}_{f' \in \mathcal{H}^{t}} |f - f'|_{\infty}$ be a function living in RKHS $\mathcal{H}^{t}$ of kernel $k_t$ that is closest to the true function $f$ measured by the infinity norm.  To avoid clutter we will  write $\mu_t = \mu_t^{g_t}$ and $\sigma_t = \sigma_t^{g_t}$.
 Let $\bm{y}^*_{t}$ be the vector of corrupted observations generated by $\hat{f}_t$, i.e. $\bm{y}^*_{t} = (\hat{f}_t(\bm{x}_1) + n_1, \dots, \hat{f}_t(\bm{x}_t) + n_t)$ for some $\sigma_n$-subgaussian noise $n_t$. Then let us define $\mu^*_{ t}$ as the Gaussian Process posterior utilising kernel $k_t$, which uses observations $\bm{y}^*_{t}$, i.e.:
\begin{equation*}
    \mu^*_{t}(\bm{x}) = \bm{k}_{t}^{\mathsf{T}}(\bm{x})\left(\bm{K}_{t} + \sigma_n^2\bm{I}\right)^{-1} \bm{y}^*_{t}
\end{equation*}
In contrast, the mean computed by the algorithm $ \mu_{t}(\bm{x})$ uses the full observations $\bm{y}_{t}= (f(\bm{x}_1) + n_1, \dots, f(\bm{x}_t) + n_t)$ from the true black-box function $f$, without knowing which part came from $\hat{f}_t$. Let us assume, we run a UCB-style BO algorithm, with an acquisition function $\alpha_t(\bm{x})$ that we will define later.
Let $\bm{x}^* = \textrm{argmax}_{\bm{x} \in \mathcal{X}}f(\bm{x})$ be the maximiser of black-box function. If we now look at the instantaneous regret, we get:
\begin{equation*}
    r_t = f(\bm{x}^*) - f(\bm{x}_t) = f(\bm{x}^*) - \max_{\bm{x} \in \mathcal{X}}\hat{f}_t(\bm{x}) + \max_{\bm{x} \in \mathcal{X}}\hat{f}_t(\bm{x}) - f(\bm{x}_t) \le
\end{equation*}
\begin{equation*}
    \le \Big(f(\bm{x}^*) - \hat{f}_t(\bm{x}^*) \Big) + \left(\hat{f}_t(\bm{x}^*) - \max_{\bm{x} \in \mathcal{X}}\hat{f}_t(\bm{x}) \right) + \left(\max_{\bm{x} \in \mathcal{X}}\hat{f}_t(\bm{x} ) - \hat{f}_t(\bm{x}_t) \right) + \left(\hat{f}_t(\bm{x}_t) - f(\bm{x}_t) \right) \le 
\end{equation*}
\begin{equation} \label{eq:inst_regret_as_ft_diff}
    \le 2\epsilon_t + \max_{\bm{x} \in \mathcal{X}}\hat{f}_t(\bm{x}) - \hat{f}_t(\bm{x}_t),
\end{equation}
where the last inequality is due to the fact that $|f(\bm{x}) - \hat{f}_t(\bm{x})| \le \epsilon_t$ and  $ \hat{f}_t(\bm{x}^*) \le \max_{\bm{x} \in \mathcal{X}}\hat{f}_t(\bm{x})$.
 Consequently, $\max_{\bm{x} \in \mathcal{X}}\hat{f}_t(\bm{x}) - \hat{f}_t(\bm{x}_t)$ becomes the new term we need to bound. We now observe the following:
\begin{equation*}
    \hat{f}_t(\bm{x}) = \mu_{t-1}(\bm{x}) + \left( \hat{f}_t(\bm{x}) - \mu^*_{t-1}(\bm{x}) \right) + \left(\mu^*_{t-1}(\bm{x}) - \mu_{t-1}(\bm{x}) \right)
\end{equation*}
Hence, we would like to bound the differences $\hat{f}_t(\bm{x}) - \mu_t^*(\bm{x})$ and $\mu_t(\bm{x}) - \mu_t^*(\bm{x})$. To do so, we recall two lemmas from existing literature. 
\begin{restatable}[Adapted Theorem 2 from \cite{chowdhury2017kernelized}]{lemma}{ucb} \label{lemma:ucb}
 Let $\mathcal{X} \subset \mathbb{R}^d$, and $\hat{f}_t : \mathcal{X} \to \mathbb{R}$ be a member of the RKHS of real-valued functions on $\mathcal{X}$ with kernel $k_t$ defined by $g_t \in \mathcal{G}$, where $g_t$ can be changed at each $t$, with
RKHS norm bounded by $\lVert \hat{f}_t \rVert_{t} \le B$. Let the observations $\bm{y}_t$ be corrupted by some $\sigma$-subgaussian noise. Then, with probability at least $1 - \delta$, the following holds for all $\bm{x} \in \mathcal{X}$:
\begin{equation*}
|\mu_{t-1}^*(\bm{x}) - \hat{f}_t(\bm{x})| \le \left( B + \sigma\sqrt{2\gamma_{t-1} + 1 + \ln(1/\delta))} \right) \sigma_{t-1}(\bm{x})
\end{equation*}
\end{restatable}
\begin{proof}
The proof is identical to the proof of Theorem 2 in \cite{chowdhury2017kernelized}, except that now the kernel depends on time.

\end{proof}

\begin{restatable}[Lemma 2 from \cite{bogunovic2021misspecified}]{lemma}{mismean} \label{lemma:mismean}
For any $\bm{x} \in \mathcal{X}$ and $t \ge 1$, we have
\begin{equation*}
|\mu_{t-1}(\bm{x}) - \mu^*_{t-1}(\bm{x})| \le \frac{\sigma_{t-1}(\bm{x})}{\sigma_n} \epsilon_t \sqrt{t}  ,
\end{equation*}
where $\epsilon_t = \min_{f' \in \mathcal{H}^{t}} |f - f'|_{\infty}$.
\end{restatable}

We adopt the notation $\hat{\bm{x}}_t= \max_{\bm{x} \in \mathcal{X}}\hat{f}_t(\bm{x})$.
Following up on Inequality \ref{eq:inst_regret_as_ft_diff}, by Lemmas \ref{lemma:ucb} and \ref{lemma:mismean} with probability at least $1 - \delta$, we get that the cumulative regret $R_T$ admits the following upper bound:
\begin{equation*}
    R_T = \sum_{t=1}^T r_t \le \sum_{t=1}^T 2\epsilon_t + \sum_{t=1}^T\hat{f}_t(\hat{\bm{x}}) - \hat{f}_t(\bm{x}_t) \le 
\end{equation*}
\begin{equation*}
    \le \sum_{t=1}^T 2\epsilon_t + \sum_{t=1}^T \mu_{t-1}(\hat{\bm{x}}_t) + \beta_t \sigma_{t-1}(\hat{\bm{x}}_t) - \left(\mu_{t-1}(\bm{x}_t) - \beta_t \sigma_{t-1}(\bm{x}_t) \right) ,
\end{equation*}
where we have introduced $\beta_t$ defined as:
\begin{equation}\label{eq:bonus}
    \beta_{t} = B + \sigma \sqrt{2(\gamma_{t-1} + 1 + \ln(1/\delta))} + \frac{ \epsilon_t \sqrt{t}}{\sigma_n} .
\end{equation}

Let us now define the acquisition rule of our BO algorithm as $\max_{\bm{x} \in \mathcal{X}}  \alpha_t(\bm{x})$, where
\begin{equation*} 
    \alpha_t(\bm{x}) =  \max_{\bm{x} \in \mathcal{X}} \mu_{t-1}(\bm{x}) + \beta_{t}\sigma_{t-1}(\bm{x})
\end{equation*}

By the acquisition rule, if point $\bm{x}_t$ was selected, then $\mu_{t-1}(\bm{x}_t) + \beta_t \sigma_{t-1}(\bm{x}_t) \ge \mu_{t-1}(\hat{\bm{x}}_t) + \beta_t \sigma_{t-1}(\hat{\bm{x}}_t)$ and so:
\begin{equation*}
    R_T \le \sum_{t=1}^T2\epsilon_t + \sum_{t=1}^T \mu_{t-1}(\hat{\bm{x}}_t) + \beta_t \sigma_{t-1}(\hat{\bm{x}}_t) - \mu_{t-1}(\bm{x}_t) + \beta_t \sigma_{t-1}(\bm{x}_t)
\end{equation*}
\begin{equation*}
    \le \sum_{t=1}^T2\epsilon_t + \sum_{t=1}^T \mu_{t-1}(\bm{x}_t) + \beta_t \sigma_{t-1}(\bm{x}_t) - \mu_{t-1}(\bm{x}_t) + \beta_t \sigma_{t-1}(\bm{x}_t) = \sum_{t=1}^T\epsilon_t + \sum_{t=1}^T 2\beta_t \sigma_{t-1}(\bm{x}_t) 
\end{equation*}
Substituting the definition of $\beta_t$ (Eq. \ref{eq:bonus}) we get:
\begin{equation*}
    R_T \le \sum_{t=1}^T2\epsilon_t + \sum_{t=1}^T2 \left(B + \sigma\sqrt{2(\gamma_{t} + 1 + \ln(1/\delta))} + \frac{ \epsilon_t \sqrt{t}}{\sigma_n} \right) \sigma_{t-1}(\bm{x}_t) \le
\end{equation*}

\begin{equation*}
      \le \sum_{t=1}^T2\epsilon_t + 2 \left(B + \sigma\sqrt{2(\gamma_{T} + 1 + \ln(1/\delta))} \right)  \sum_{t=1}^T\sigma_{t-1}(\bm{x}_t) +  \frac{\sqrt{T}}{\sigma_n} \sum_{t=1}^T\epsilon_t\sigma_{t-1}(\bm{x}_t) 
\end{equation*}
We now observe that by Cauchy-Schwarz we have $\sum_{t=1}^T\sigma_{t-1}(\bm{x}_t) \le \sqrt{T \sum_{t=1}^T\sigma_{t-1}^2(\bm{x}_t)}$ and \\
$\sum_{t=1}^T\epsilon_t\sigma_{t-1}(\bm{x}_t) \le \sqrt{\sum_{t=1}^T\epsilon^2_t } \sqrt{\sum_{t=1}^T\sigma^2_{t-1}(\bm{x}_t)}$. We thus obtain the following bound:
 \begin{equation*}
     R_T \le \sum_{t=1}^T2\epsilon_t + 2 \left(B + \sigma \sqrt{2(\gamma_{T} + 1 + \ln(1/\delta))} \right)  \sqrt{T \sum_{t=1}^T\sigma^2_{t-1}(\bm{x}_t) } +  \frac{ 1}{\sigma_n}\sqrt{\sum_{t=1}^T \epsilon_t^2}\sqrt{T \sum_{t=1}^T\sigma^2_{t-1}(\bm{x}_t) }
\end{equation*}

  Observe that $\sigma^2_{t-1}(\bm{x}_t) \le \sigma^2_n C \log (1 + \sigma^{-2}_n \sigma_{t-1}^2(\bm{x}_t))$, where $C = \sigma^{-2}_n / \log(1 + \sigma^{-2}_n )$ and by Lemma 5.3 of \cite{srinivas2009gaussian}, we have $\sum_{t=1}^T \log (1 + \sigma^{-2}_n \sigma_{t-1}^2(\bm{x}_t)) \le 2 \gamma_t$. We thus obtain:
\begin{equation*}
    R_T \le \sum_{t=1}^T2\epsilon_t + 2 \left( B + \sigma\sqrt{2(\gamma_T + 1 + \ln(1/\delta))} \right)  \sqrt{T \gamma_T} +  \frac{\sqrt{T \gamma_T}}{\sigma_n}\sqrt{\sum_{t=1}^T \epsilon_t^2} 
\end{equation*}
\begin{equation*}
     R_T =  \mathcal{O} \left ( \sqrt{T\gamma_T} \left( B  + \sqrt{\ln\frac{1}{\delta} + \gamma_T }   + \sum_{t=1}^T \epsilon_t \right)\right ) ,
\end{equation*}
 since due to $\forall_t \epsilon_t \ge 0$, we have $ \sqrt{\sum_{t=1}^T \epsilon_t^2} \le \sum_{t=1}^T\epsilon_t$.

\end{proof}
\subsection{Proof of Corollary \ref{Corr:Regret}}
\label{ap:proof_corrregret}
\finalregretbound*
\begin{proof}

    Using Markov's inequality we have: $
        \mathbb{P}\text{r} \left ( \sum_{t=1}^T \epsilon_t \ge \mathbb{E}_S \left[\sum_{t=1}^T \epsilon_t\right] / \delta_B \right) \le \delta_B$.
    We now combine this fact with the bound developed in Theorem \ref{th:regretbound} and combine the probabilities using union bound to arrive at the corollary's statement. 
    \end{proof}
\subsection{Proof of Proposition \ref{prop:informationgain}}   \label{ap:proof_informationgain}
\informationgain*
\begin{proof}
    As shown in \cite{rolland2018high}, for an additive squared exponential kernel in $A$ dimensional space, such that each subkernel operates on at most $B$ dimensions, we get that the maximum information gain is bounded as:
    \begin{equation*}
        \gamma_T \le AB^B \log T ^{B + 1}
    \end{equation*}
    We observe that in our case $A = d$ and $B = 2$, which finishes the proof.
\end{proof}

\subsection{Proof of Proposition \ref{prop:lowerinformation}}
\label{ap:proof_lowerinformation}
\lowerinformation*
\begin{proof}
    We will denote by $\mathcal{G}$ the class of all possible tree decompositions in $d$ dimensions. As such $\bigcup_{g \in \mathcal{G}}$ is the set of all pairwise components. Additionally, we will introduce the following notation:
    \begin{equation*}
        f^{\textrm{full}} = \sum_{c \in \bigcup_{g \in \mathcal{G}}} f_c \quad \quad y^{\textrm{full}}_t = f^{\textrm{full}}(\bm{x}_t) + \epsilon_t
    \end{equation*}
    \begin{equation*}
        \hat{f}_t = \sum_{c \in  g_t} f_c \quad  \quad y_t = \hat{f}_t(\bm{x}_t) + \epsilon_t ,
    \end{equation*}
    where the quantities have the following distribution for all $c \in \bigcup_{g \in \mathcal{G}}$ and all $t >0$:
    \begin{equation*}
        f_c \sim \mathcal{GP}(0, k(\bm{x}, \bm{x}')) \quad \quad \epsilon_t \sim \mathcal{N}(0, \sigma^2_n).
    \end{equation*}
    For some fixed sequence $\bm{X}_T = (\bm{x}_1, \dots, \bm{x}_T)$, we will write $\bm{f}_T = (f_1(x_1), \dots, \hat{f}_T(x_T))$ and $\bm{y}_T = (y_1, \dots, y_T)$ and equivalently for $\bm{f}^{\textrm{full}}$ and $\bm{\tilde{y}}_T$. By properties of mutual information, we have:
    \begin{equation*}
        I(\bm{y}^{\textrm{full}}_T, \bm{f}^{\textrm{full}}) - I(\bm{y}_T, \bm{f}_T) = H(\bm{y}^{\textrm{full}}_T) - H(\bm{y}^{\textrm{full}}_T|\bm{f}^{\textrm{full}}) - H(\bm{y}_T) + H (\bm{y}_T| \bm{f}_T) =
        H(\bm{y}^{\textrm{full}}_T) - H(\bm{y}_T),
    \end{equation*}
    where the last inequality is true, as the conditional distributions are the same. Using the formula for the entropy of multivariate Gaussian we get:
    \begin{equation*}
        H(\bm{y}^{\textrm{full}}_T) - H(\bm{y}_T) = \frac{1}{2} \ln \left( \frac{\det(\mathcal{I}\sigma^2_n + K_T^{\textrm{full}})}{\det(\mathcal{I}\sigma^2_n + K_T)} \right) = \frac{1}{2} \sum_{t=1,\dots,T} \ln \left( \frac{\sigma^2_n + \lambda_t(K_T^{\textrm{full}})}{\sigma^2_n + \lambda_t(K_T)} \right) ,
    \end{equation*}
    where $\lambda_t(A)$ means the $t$-th  largest eigenvalue of $A$ and the covariance matrices are defined as $(K_T^{\textrm{full}})_{i,j} =  \sum_{c \in \bigcup_{g \in \mathcal{G}}} k_c(x_i, x_j)$ and $(K_T)_{i,j} =  \sum_{c \in g_i \cap g_j} k_c(x_i, x_j)$. We can thus write $K_T^{\textrm{full}} = K_T + \dot{K}_T$, where $(\dot{K}_T)_{i,j} = \sum_{c \notin g_i \cap g_j} k_c(x_i, x_j)$. One can easily see that $\dot{K}_T$ must be PSD, as such we have $\lambda_t(K_T^{\textrm{full}}) \ge \lambda_t(K_T) $, as adding a PSD matrix to another PSD matrix can never decrease its eigenvalues. This gives us:
    \begin{equation*}
        \frac{1}{2} \sum_{t=1,\dots,T} \ln \left( \frac{\sigma^2_n + \lambda_t(K_T^{\textrm{full}})}{\sigma^2_n + \lambda_t(K_T)} \right) \ge \frac{1}{2} \sum_{t=1,\dots,T} \ln \left( 1 \right) = 0,
    \end{equation*}
    with equality if and only if $\forall_{t=1,\dots,T}\lambda_t(K_T^{\textrm{full}}) = \lambda_t(K_T)$, which can happen only if all eigenvalues of $\dot{K}_T$ are zero, meaning $\forall_{t>0} g_t =  \bigcup_{g \in \mathcal{G}}$, i.e. all components are included at every step. This can only happen if there is only one decomposition in the class $|\mathcal{G}|$. As our class $\mathcal{G}$ consist of trees, this is only possible when $d \le 2$. 
\end{proof}

\subsection{Proof of Theorem \ref{th:optimalscheme}} \label{ap:proof_optimalscheme}
\optimalscheme*
\begin{proof}
     We start by introducing the following lemma.
    \begin{restatable}[]{lemma}{sum_of_mismatch} \label{lemma:sum_of_mismatches}
    Let $g$ be the true tree-based decomposition of $f$ and $\tilde{g} = g \setminus \Tilde{C}$ for some set of omitted pair-wise components $\Tilde{\mathcal{C}}$. Then we have:
    \begin{equation*}
        \min_{f' \in \mathcal{H}^{\tilde{g}}}|f - f'|_{\infty} \le \sum_{c \in \Tilde{C}} M_c  
    \end{equation*}
    \end{restatable}
    \begin{proof}
        \begin{equation*}
            \min_{f' \in \mathcal{H}^{\tilde{g}}}|f - f'|_{\infty} = \min_{f' \in \mathcal{H}^{\tilde{g}}}|\sum_{c \notin \Tilde{\mathcal{C}}}f_c + \sum_{c \in \Tilde{\mathcal{C}}}f_c- f'|_{\infty} \le \min_{f' \in \mathcal{H}^{\tilde{g}}}|\sum_{c \notin \Tilde{\mathcal{C}}}f_c- f'|_{\infty} + \sum_{c \in \Tilde{\mathcal{C}}}|f_c |_{\infty} \le \sum_{c \in \Tilde{\mathcal{C}}} M_c 
        \end{equation*}
    where the first inequality is due to the triangle inequality and second due to the fact that  $|f_c |_{\infty} \le M_c$ and $\sum_{c \notin \Tilde{\mathcal{C}}}f_c \in \mathcal{H}_{\tilde{g}}$.
    \end{proof}
    Let us now define the probability of choosing a decomposition $g$ at time $t$ by our decomposition proposing scheme as $P^{\mathcal{G}}_t(g)$. A deterministic scheme will just be a special case of the probabilistic scheme, where all probability is concentrated on one decomposition. Instead of thinking about the adversary as selecting a function $f$, we can think about them as selecting the norm parameters for each pair-wise component $M_c$, with the constraint that $\sum_{c \in g} M_c \le M $. Since the adversary knows the scheme, in the worst case they can select the decomposition and function so that the expected mismatch is maximal. This corresponds to:
    \begin{equation*}
        \max_{g, f} \mathbb{E}_S\left[\sum_{t=1}^T \epsilon_t \right] = \max_{g, f} \sum_{g' \in \mathcal{G}}\sum_{t=1}^T \min_{f' \in \mathcal{H}^{g'}}|f - f'|_{\infty} P^{\mathcal{G}}_t(g') = \max_{g, \{M_{c}\}_{c \in g}} \sum_{t=1}^T \sum_{g' \in \mathcal{G}} \sum_{c \in g}M_c \mathbf{1}_{c \notin g'} P^{\mathcal{G}}_t(g') ,
    \end{equation*}
    where the last equality is due to Lemma \ref{lemma:sum_of_mismatches} and the fact that the adversary, in the worst case, will choose a function with the highest possible mismatch. We can now exchange the order of summation to obtain:
    \begin{equation*}
         \max_{g, f} \mathbb{E}_S\left[\sum_{t=1}^T \epsilon_t \right] = \max_{g, \{M_{c}\}_{c \in g}}  \sum_{c \in g}M_c    \sum_{g' \in \mathcal{G}}\mathbf{1}_{c \notin g'} \sum_{t=1}^T P^{\mathcal{G}}_t(g') = \max_{g, \{M_{c}\}_{c \in g}}  \sum_{c \in g} M_{c}\mathbb{E}_S[ N_{\neg c}] ,
    \end{equation*}
    where $\mathbb{E}_S[ N_{\neg c}] =\sum_{g' \in \mathcal{G}}\mathbf{1}_{c \notin g'} \sum_{t=1}^T P^{\mathcal{G}}_t(g')$ is the expected number of time that a pair-wise component $c$ is \textbf{not} included in the proposed decomposition. Note that the expression above is maximised subject to the constraint that $\sum_{c \in \mathcal{C}}M_c \le M$. Thus the maximum is achieved when the biggest mismatch is placed on the pair-wise component that is on average least selected. Formally, let $c^* = \arg \max_{c \in \mathcal{C}} \mathbb{E}_S[ N_{\neg c}]$, then the solution to the constrained maximization problem above is $M_c = M \mathbf{1}_{c=c^*}$ and $g$ can be any decomposition including $c^*$. We thus obtain:
    \begin{equation*}
        \max_{g, f} \mathbb{E}_S\left[\sum_{t=1}^T \epsilon_t \right] = M \max_{c \in \mathcal{C}}\mathbb{E}_S[ N_{\neg c}].
    \end{equation*}
    This is minimised, when for the selected scheme $S$ the expected number of times the least selected pair-wise component is not selected is minimal. This happens when the chance to include each of the  pair-wise components is the same. Since all of the decompositions have the same number of pair-wise components, this corresponds to a uniform distribution over all trees. This proves the claim that a uniformly random scheme achieves the smallest expected mismatch. Under the uniformly random scheme, we select $E$ pair-wise components at each time out of all possible $d(d-1)/2$, so the probability of any one component being selected is $\frac{2E}{d(d-1)}$ and the inverse event has a probability of $1 - \frac{2E}{d(d-1)}$. Since the scheme is the same across all timesteps we get $\mathbb{E}_S[ N_{\neg c}] = T\left(1 - \frac{2E}{d(d-1)}\right)$.
   
\end{proof}
\newpage
\section{Procedure for Sampling random Trees} \label{ap:treesampler}
\begin{algorithm}[h]
 \caption{Random Tree Sampler}
   \label{alg:samplingtrees}
    \begin{algorithmic}[1]
        \STATE {\bfseries Input:} \# of edges $E$, dimensionality of the problem $d$
        \STATE Set $L = [1, \dots, d]$
        \STATE Create $L_{\text{in}}$ and $L_{\text{out}}$ by randomly permuting $L$
        \STATE Initialise a Union-Find structure $\texttt{UF}$ \& empty graph $g$
        \FOR{$n_{\text{in}} \in L_{\text{in}}$}
            \FOR{$n_{\text{out}} \in L_{\text{out}}$}
                \IF{!$\texttt{UF}.\texttt{connected}$($n_{\text{in}}$, $n_{\text{out}}$)}
                    \STATE $\texttt{UF.union}$($n_{\text{in}}$, $n_{\text{out}}$)
                    \STATE $g$.$\texttt{add\_edge}$($n_{\text{in}}$, $n_{\text{out}}$)
                \ENDIF
                \IF{ $g$.$\texttt{number\_of\_edges}$() $= E$}
                    \STATE  {\bfseries Return} $g$
                \ENDIF
            \ENDFOR
        \ENDFOR
    \end{algorithmic}
\end{algorithm}

\newpage
\section{Algorithm Settings} \label{ap:algo_setting}
In Table \ref{tab:general_hyperparam}, we detail settings used by each algorithm. Those values are used for all of the experiments. 

\begin{table}[h]
    \centering
    \begin{tabular}{c|c|c}
        Algorithm & Setting & Value  \\
        \midrule
        Tree & Acquisition function &  Additive UCB with $ \beta_t = 0.5 \log(2t)$\\
         & Decomposition learning interval & 15\\
             & Gibbs sampling iterations &  100  \\
        \midrule 
         RDUCB & Acquisition function &  Additive UCB with $ \beta_t = 0.5 \log(2t)$\\
         & Size of random tree &  $\max \{\lfloor \sfrac{d}{5} \rfloor, 1\}$\\
        \midrule
        HeSBO & Acquisition function &  EI \cite{nguyen2017regret}\\
         & Size of embedding & $\sqrt{d}$ \\
         \midrule
         REMBO/ CoordinateLineBO & All & Default values from  \hyperlink{https://github.com/kirschnj/LineBO/blob/master/config/hartmann6.yaml}{here}\\
    \end{tabular}
    \caption{Hyperparameters used by algrotihms for all experiments. $t$ denotes the timestep and $d$ dimensionality of the problem.}
    \label{tab:general_hyperparam}
    
\end{table}

\subsection{Computing Resources}
All experiments were run on machines with specifications described in Table \ref{tab:compute}.

\begin{table}[h!]
    \centering
    \begin{tabular}{l|l}
        \textbf{Component} & \textbf{Description} \\
        \hline
        CPU & Intel Core i9-9900X CPU @ 3.50GHz \\
        GPU & Nvidia RTX 2080  \\
        Memory & 64 GB DDR4 \\
    \end{tabular}
    \caption{Specifications of machines used to run experiments.}
    \label{tab:compute}
\end{table}

\newpage

\section{Toy problem details} \label{App:toy_problem} 
In this section, we describe the details of the toy problem introduced in Section \ref{Sec:Th}. For this experiment, we use algorithm setting as per Table \ref{tab:general_hyperparam}. The function we use is three-dimensional, where the last dimension is redundant. Thus, in Figure \ref{fig:toyproblem} we only plot it as a function of two variables. We chose to add one redundant dimension as otherwise, RDUCB will always be sampling the same decomposition (size of random tree $E = 1$). The formula for the function is given below:
\begin{equation*}
    f(x,y,z) = w_1\mathcal{N}_{x,y}(\mu_1, \Sigma_1) + w_2\mathcal{N}_{x}(\mu_2, \Sigma_2) + w_3\mathcal{N}_{y}(\mu_3, \Sigma_3),
\end{equation*}
where $\mathcal{N}_{c}(\mu, \sigma)$ is a $|c|$-dimensional Gaussian PDF defined on dimensions in $c$ with mean $\mu$ and covariance matrix $\Sigma$. For the toy problem, we used the numerical values shown below.
\begin{equation*}
     w_1 = 1/6 \quad w_2 = w_3 = 2.5/6
\end{equation*}
\begin{equation*}
     \mu_1 = (800, 800)^T \quad \mu_2 = \mu_3 = (300)
\end{equation*}
\begin{equation*}
     \Sigma_1 = \begin{pmatrix}
     20000 & 15000 \\
     15000 & 20000
     \end{pmatrix} \quad
     \Sigma_2 = \Sigma_3 = (
     10000
     )
\end{equation*}
Thus, there are two local optima for $x$ and $y$, suboptimal  at $(300, 300)$ and global at $(800, 800)$. Variable $z$ can be set to any value, as it does not affect the function output. The initial points given to both Tree and RDUCB were exactly the same. We show them in Figure \ref{fig:init_design_toy_problem} below.

\begin{figure}[h]
    \centering
    \includegraphics{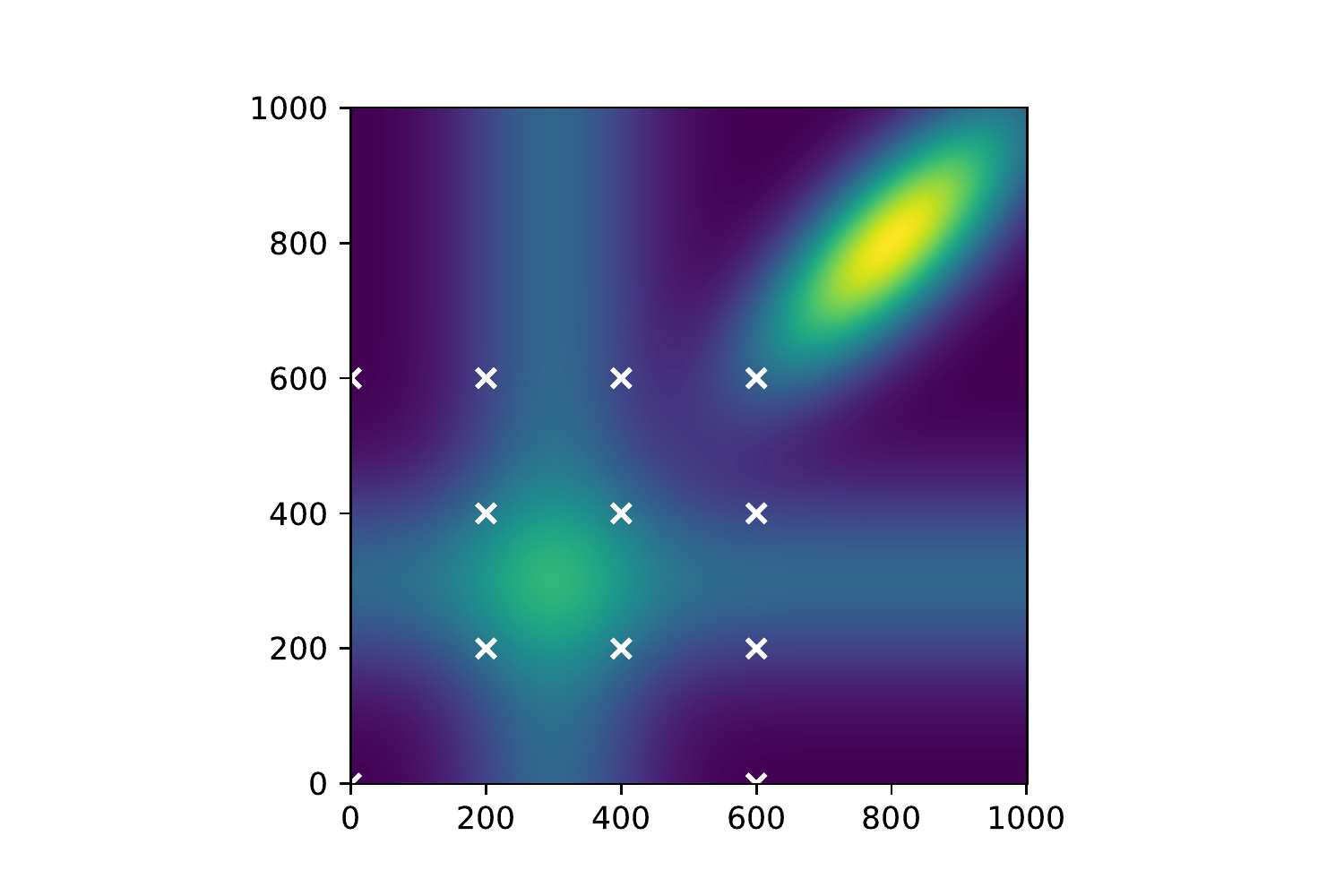}
    \caption{Initial points given to both Tree and RDUCB on the toy problem.}
    \label{fig:init_design_toy_problem}
\end{figure}
\newpage

\section{Additional Experimental Results} \label{App:Exps}

\begin{figure*}[h!]
    \begin{subfigure}[b]{0.31\textwidth}
  \centering
   \includegraphics[width=\linewidth]{figs/stybtang250.pdf}
  \caption{Stybtang250}
  \label{fig:stybtang250}
\end{subfigure}%
\begin{subfigure}{0.31\textwidth}
  \centering
   \includegraphics[width=\linewidth]{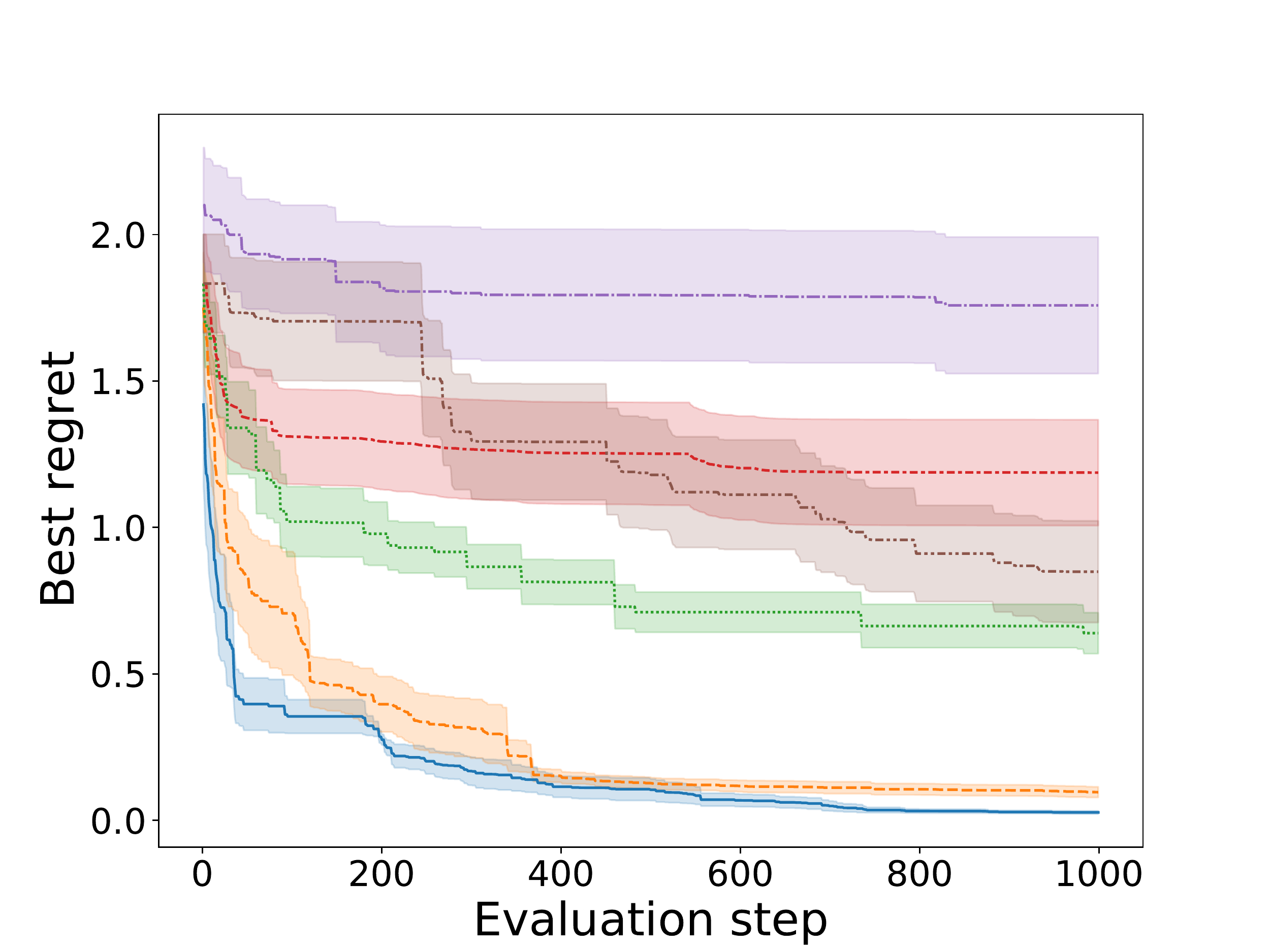}
  \caption{Hartmann6+14}
  \label{fig:hartmann614}
\end{subfigure}%
\begin{subfigure}{0.31\textwidth}
  \centering
   \includegraphics[width=\linewidth]{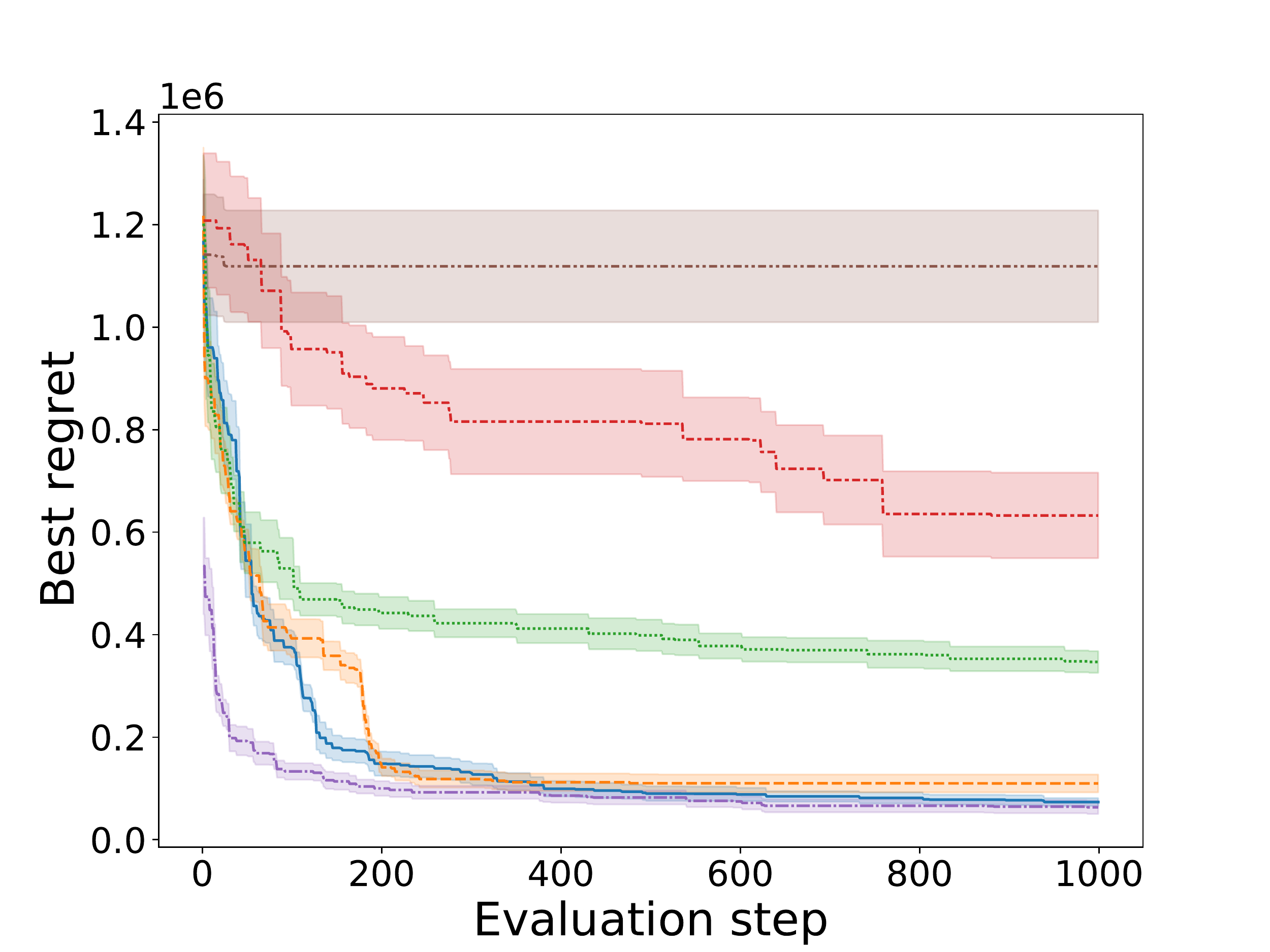}
  \caption{Rosenbrock20}
  \label{fig:rosenbrock20}
\end{subfigure}
\caption{Performance comparison on selected synthetic functions. Solid lines are the mean values over 10 seeds, and shaded areas correspond to standard error.}
 \label{fig:synthethic}
\end{figure*}

\begin{figure*}[h!]
    \begin{subfigure}[b]{0.31\textwidth}
  \centering
   \includegraphics[width=\linewidth]{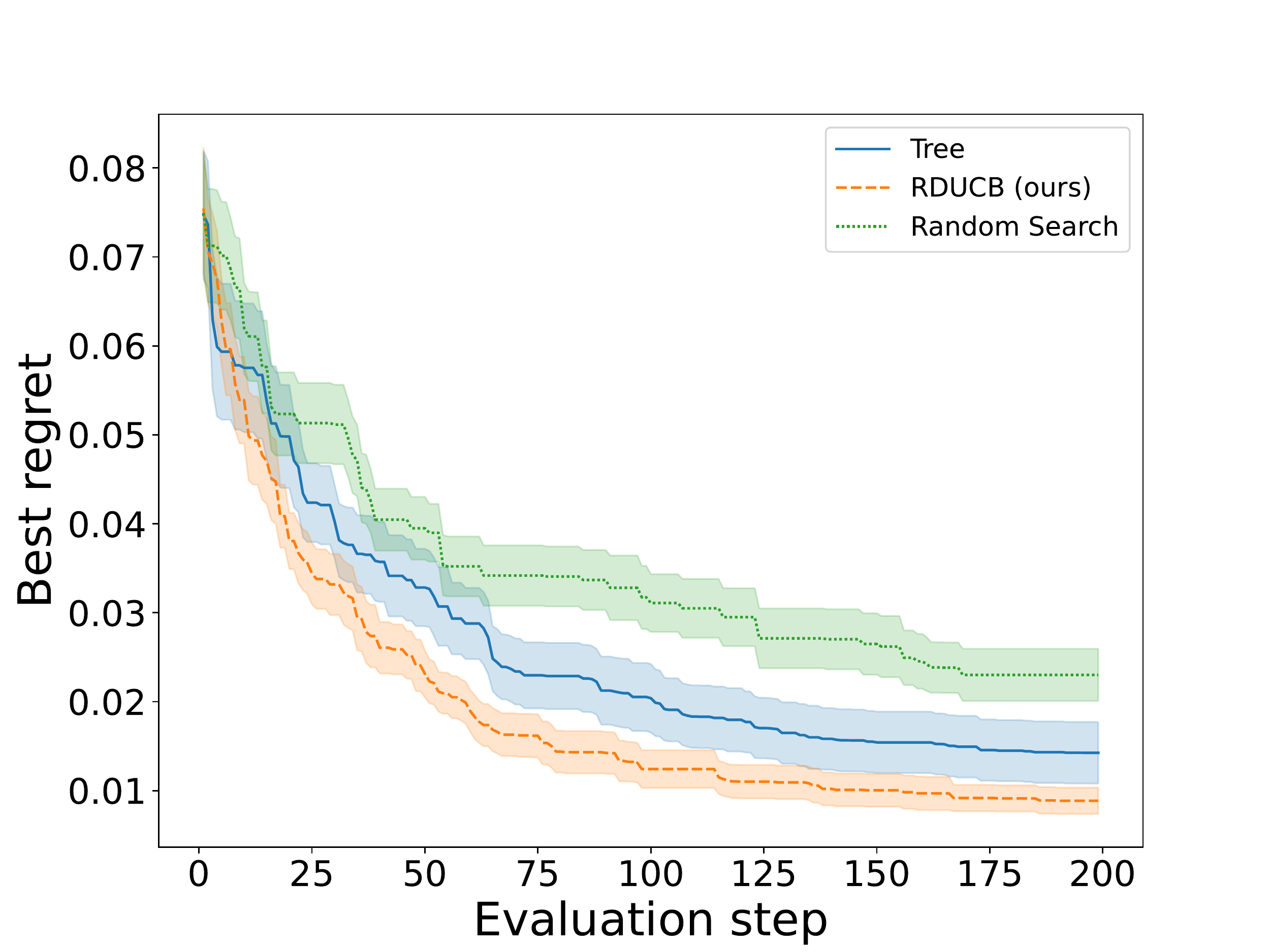}
  \caption{Protein}
  \label{fig:protein}
\end{subfigure}%
\begin{subfigure}{0.31\textwidth}
  \centering
   \includegraphics[width=\linewidth]{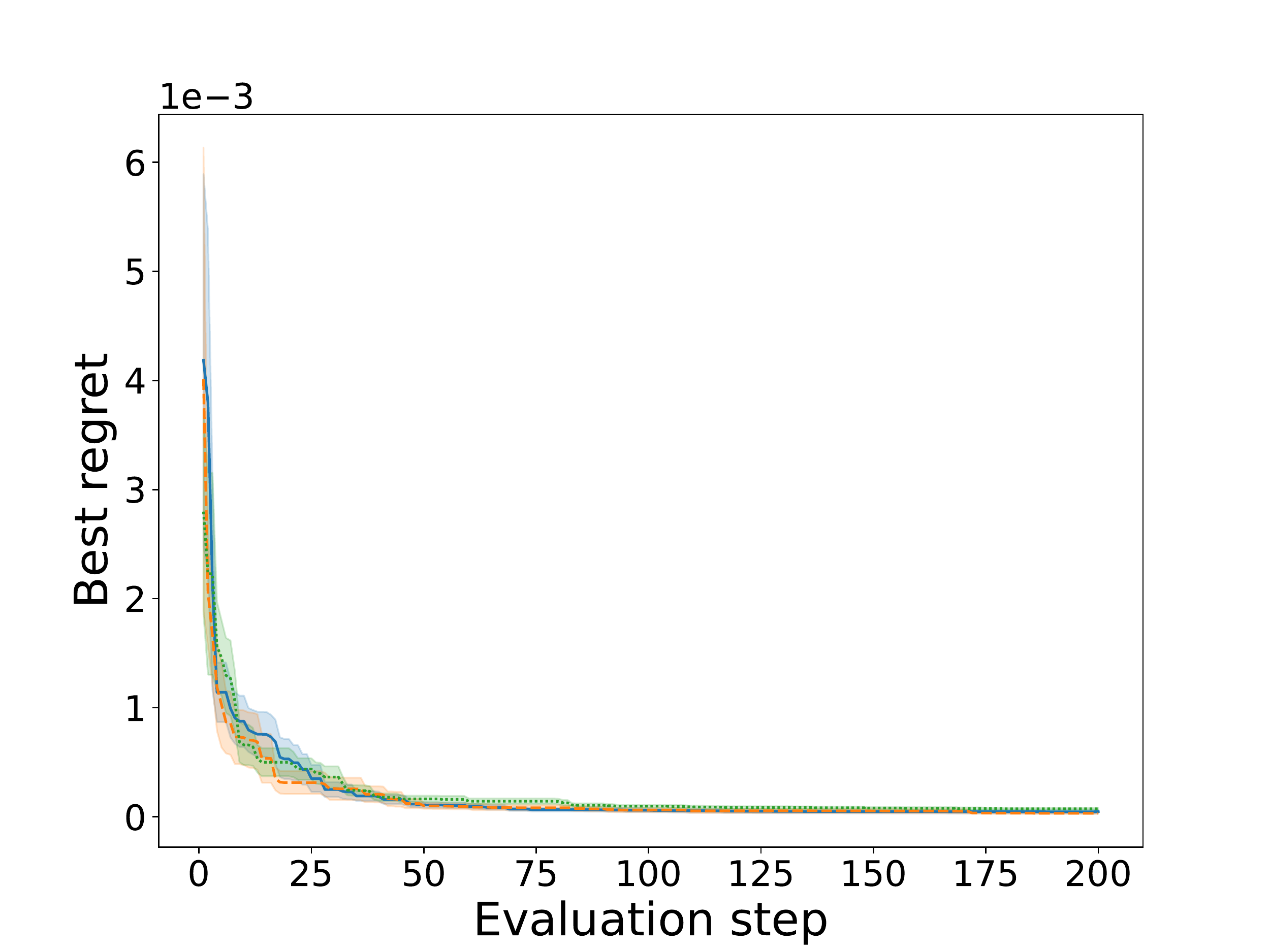}
  \caption{Naval propultion}
  \label{fig:np}
\end{subfigure}\\
\begin{subfigure}{0.31\textwidth}
  \centering
   \includegraphics[width=\linewidth]{figs/nas_pt.pdf}
  \caption{Parkinson telemonitoring}
  \label{fig:pt}
\end{subfigure}%
\begin{subfigure}{0.31\textwidth}
  \centering
   \includegraphics[width=\linewidth]{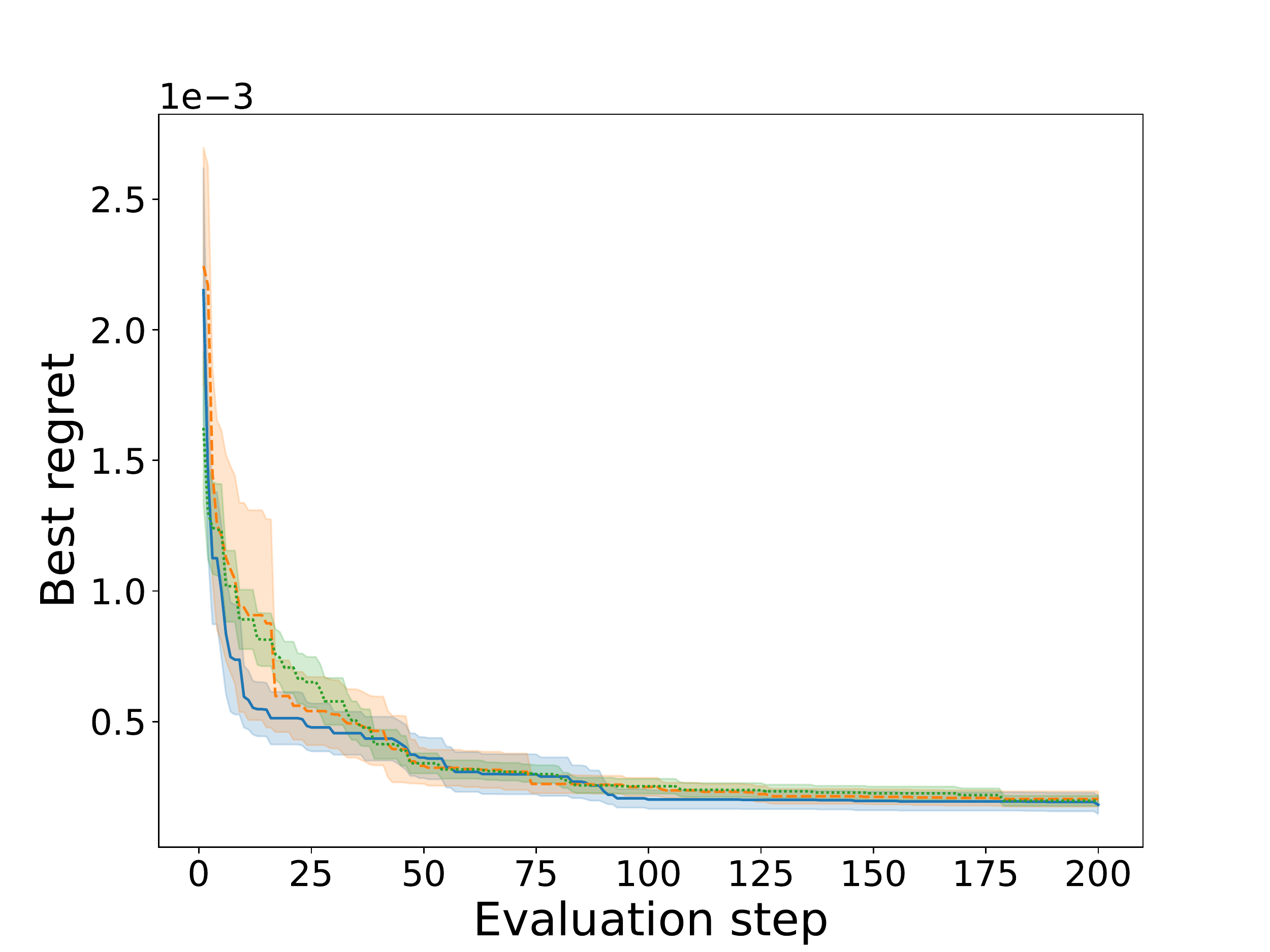}
  \caption{Slice Locatization}
  \label{fig:sl}
\end{subfigure}
\caption{Performance comparison on NAS benchmarks. Solid lines are the mean values over 20 seeds, and shaded areas correspond to standard error.}
\label{fig:nas}
\end{figure*}

\begin{figure*}[h!]
    \begin{subfigure}[b]{0.31\textwidth}
  \centering
   \includegraphics[width=\linewidth]{figs/lp_qiu.pdf}
  \caption{qiu}
  \label{fig:lp_qiu}
\end{subfigure}%
\begin{subfigure}{0.31\textwidth}
  \centering
   \includegraphics[width=\linewidth]{figs/lp_misc05inf.pdf}
  \caption{misc05inf}
  \label{fig:lp_misc05inf}
\end{subfigure}%
\begin{subfigure}{0.31\textwidth}
  \centering
   \includegraphics[width=\linewidth]{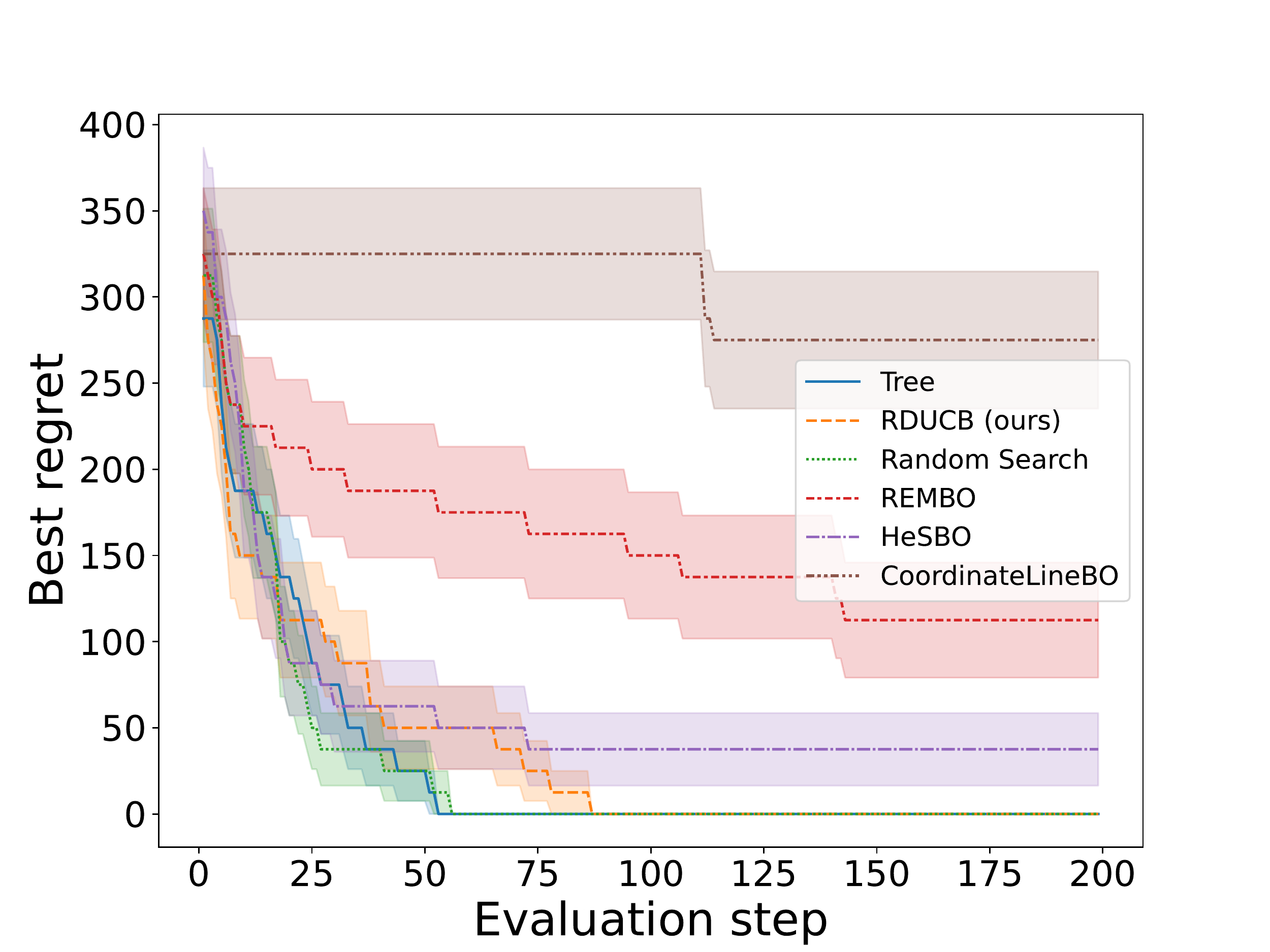}
  \caption{mtest4ma}
  \label{fig:lp_mtest4ma}
\end{subfigure}
\caption{Performance comparison on selected MIP hyperparameter tuning problems. Solid lines are the mean values over 40 seeds and shaded areas correspond to standard error.}
\label{fig:lp}
\end{figure*}

\begin{figure*}
\begin{subfigure}[b]{0.31\textwidth}
  \centering
   \includegraphics[width=\linewidth]{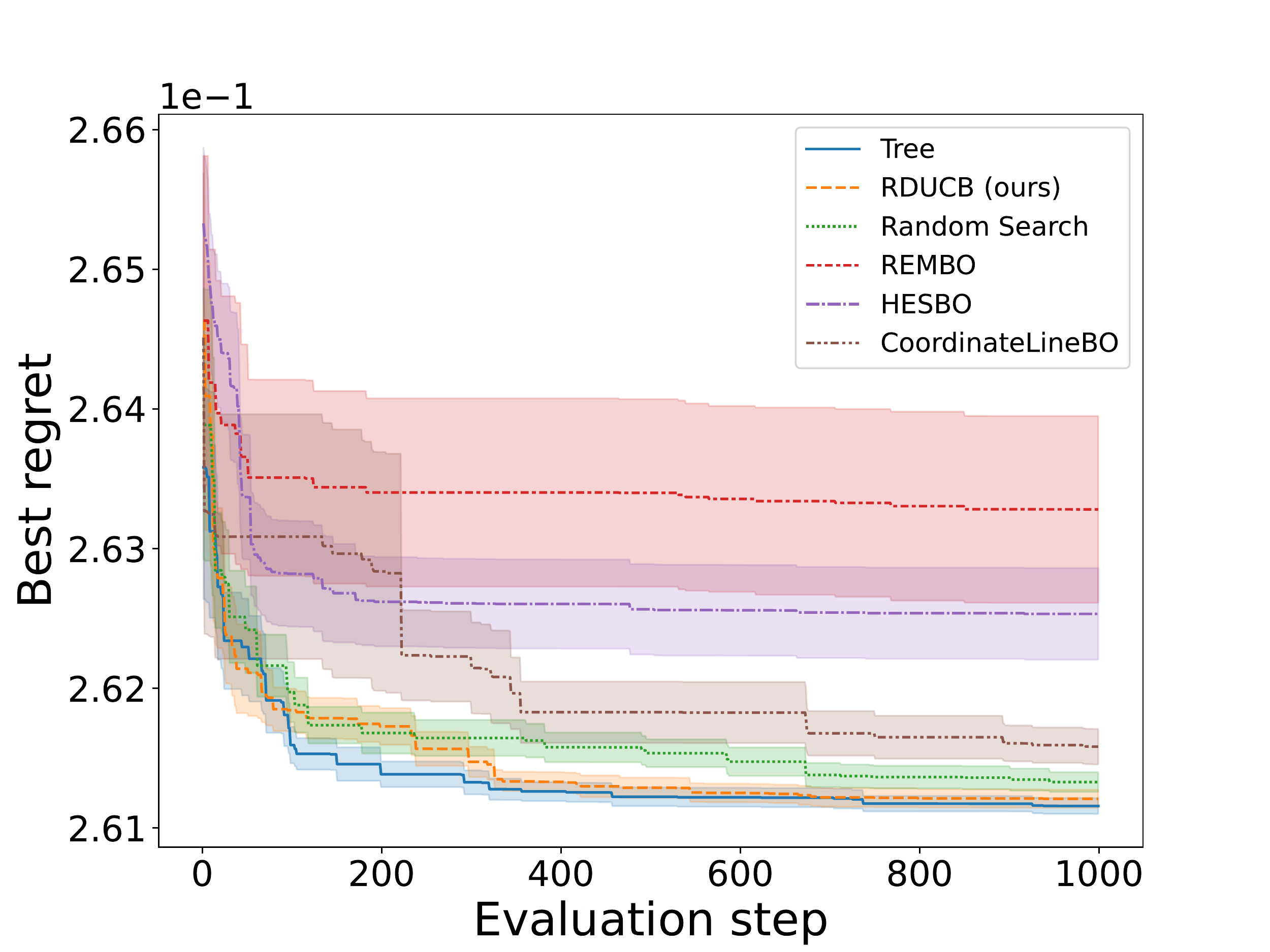}
 \caption{Breast Cancer (10 dim.)}
\end{subfigure}%
\begin{subfigure}[b]{0.31\textwidth}
  \centering
   \includegraphics[width=\linewidth]{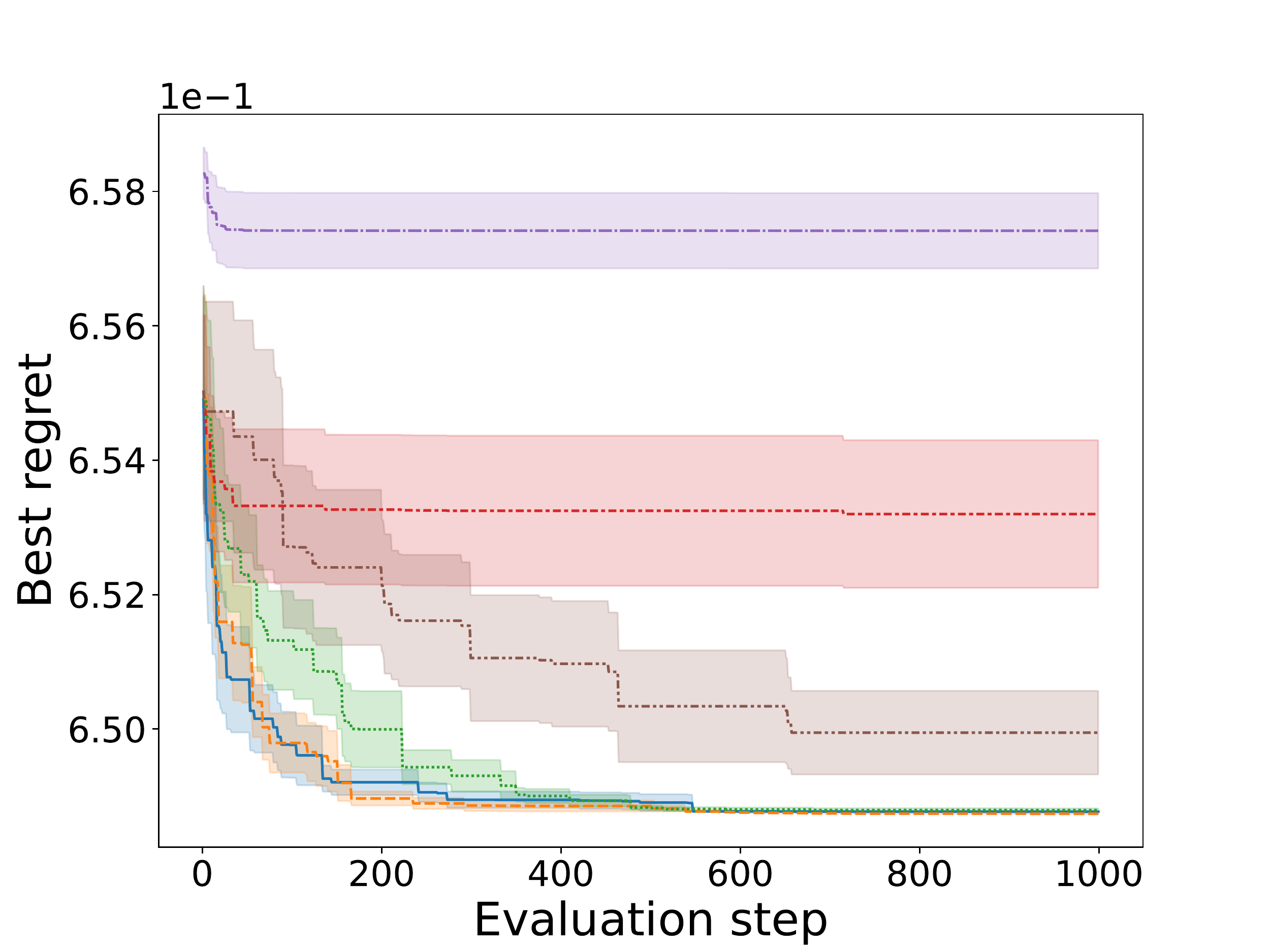}
 \caption{Diabetes (8 dim.)}
\end{subfigure}%
    \begin{subfigure}[b]{0.31\textwidth}
  \centering
   \includegraphics[width=\linewidth]{figs/lasso_dna.pdf}
 \caption{DNA (180 dim.)}
\end{subfigure}%

\caption{Performance comparison on LassoBench problems.}
 \label{fig:lasso_dna}
\end{figure*}
\clearpage
\section{Comparing RDUCB to Tree with different acquisition functions} 

\begin{figure*}[h!]
    \begin{subfigure}[b]{0.5\textwidth}
  \centering
   \includegraphics[width=\linewidth]{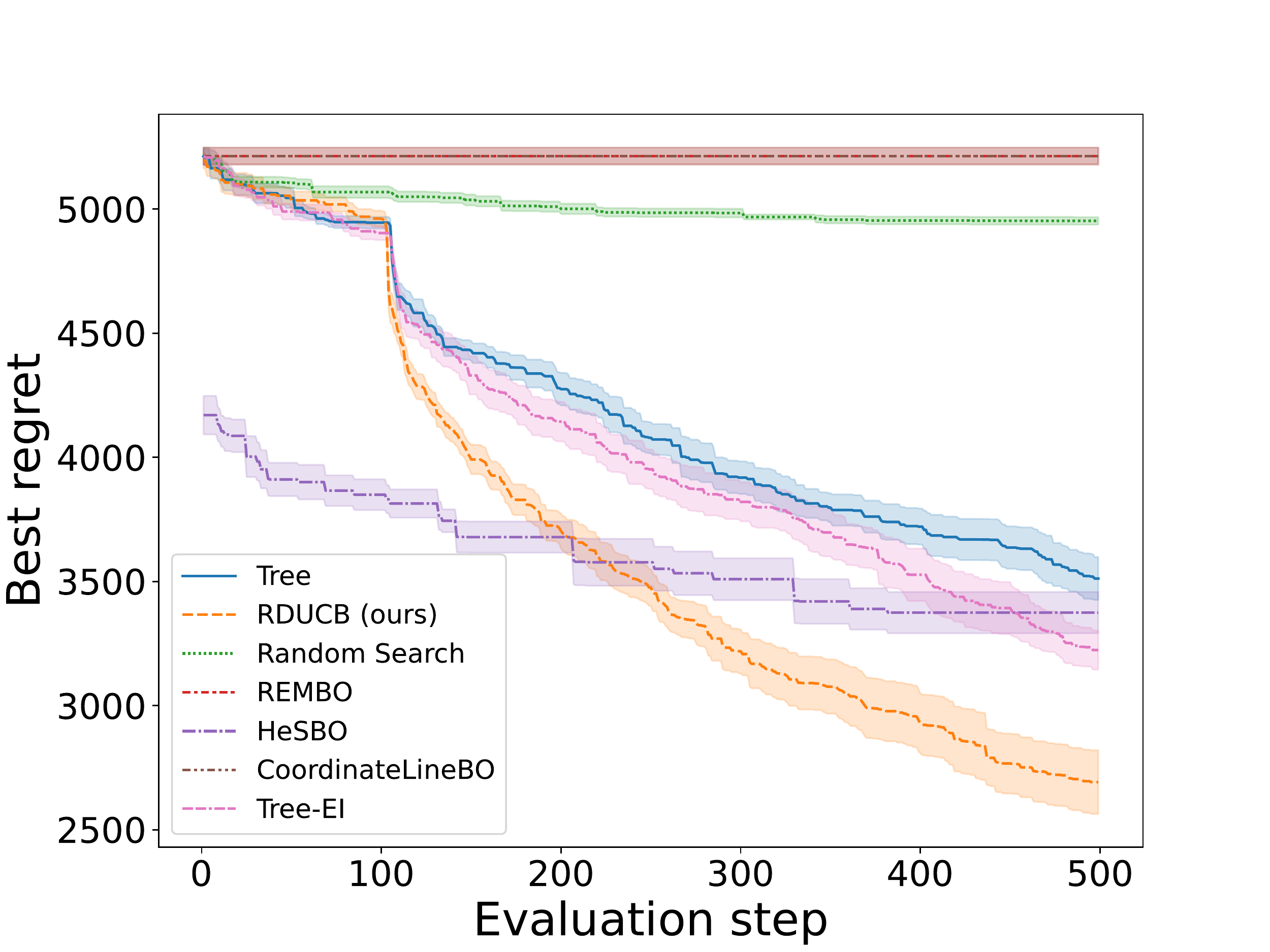}
  \caption{Stybtang250}
  \label{fig:stybtang250_ei}
\end{subfigure}%
\begin{subfigure}{0.5\textwidth}
  \centering
   \includegraphics[width=\linewidth]{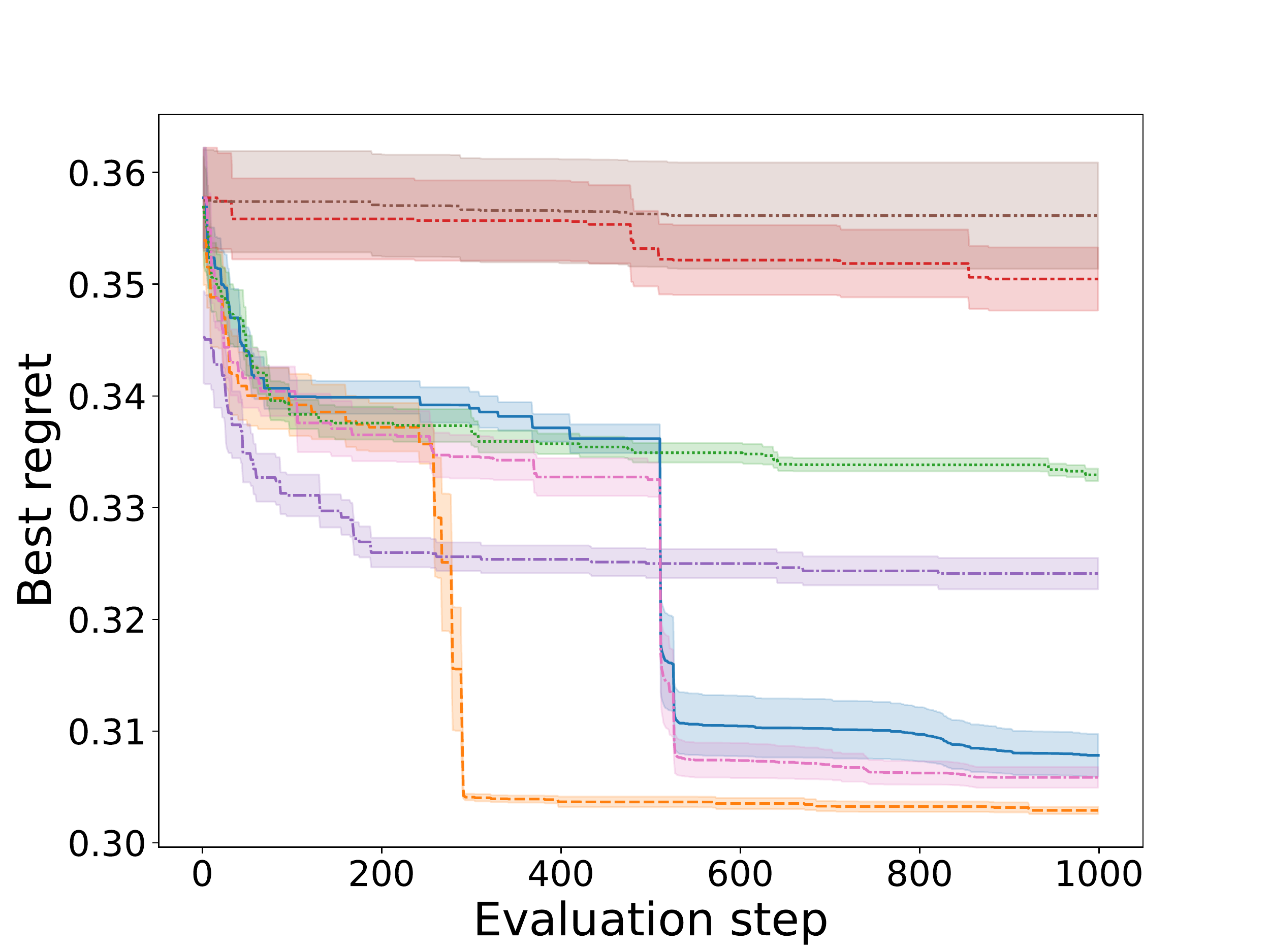}
  \caption{Lasso DNA}
  \label{fig:lasso_dna_ei}
\end{subfigure}%

\caption{Additional experiments over two highest dimensional tasks, comparing the improvement brought by RDUCB to the improvement brought by using a different acquisition function. Tree-EI stands for the Tree algorithm \cite{han2021high} utilising the expected-improvement acquisition function \cite{nguyen2017regret}. We develop an additive version of the expected improvement, where the term corresponding to component $c$ is given by $\alpha_c(\bm{x}|\mathcal{D}_{t-1}) = (\mu_{t-1, c}(\bm{x}) - \mu_{t-1}(\bm{x^+})) \Phi(\frac{\mu_{t-1, c}(\bm{x}) - \mu_{t-1, c}(\bm{x}^+)}{\sigma_{t-1, c}(\bm{x})}) + \sigma_{t-1, c}(\bm{x})\phi(\frac{\mu_{t-1, c}(\bm{x}) - \mu_{t-1, c}(\bm{x}^+)}{\sigma_{t-1, c}(\bm{x})})$, in which $x^+$ is the best point found so far and $\Phi()$ and $\phi()$ are Gaussian CDF and PDF, respectively. Shaded areas correspond to standard errors over 10 seeds. We see that although choosing a different acquisition function might bring some improvement, it is smaller compared to the improvement delivered by RDUCB.}
\label{fig:ei}
\end{figure*}

\end{document}